%% file: main.tex
\documentclass{article}

    \PassOptionsToPackage{numbers, compress}{natbib}

\usepackage[preprint]{neurips_2025}

\usepackage[utf8]{inputenc} 
\usepackage[T1]{fontenc}    
\usepackage{url}            
\usepackage{booktabs}       
\usepackage{amsfonts}       
\usepackage{nicefrac}       
\usepackage{microtype}      
\usepackage{xcolor}         

\input{packages}
\input{commands.tex}

\input{glossary}
\input{local_glossary}
\input{config}
\input{acronyms}

\usepackage{graphicx} 
\usepackage{tikz}
\tikzstyle{state}=[
    circle,
    minimum size = 10mm,
    draw=black,
    very thick,
]
\tikzstyle{midstate}=[
    circle,
    minimum size = 1mm,
    draw=black,
    thick,
]
\tikzstyle{smallstate}=[
minimum size = 0mm,
]
\tikzset{
    point/.style = {circle, draw, inner sep=0.05cm,fill,node contents={}},
}
\usetikzlibrary{decorations.text, fit, bayesnet}
\usetikzlibrary{arrows.meta}

\usepackage[american,siunitx]{circuitikz}
\usetikzlibrary{arrows,shapes,calc, positioning, patterns, decorations, decorations.markings,quotes}
\tikzset{rotarrow/.pic={
\draw[thin,->] (-0.2,-0.2)  to [out=-60,in=60, looseness=4] ++(0,0.4) node [above=1mm] {\tikzpictext};
},
}

\title{
Skill Learning via Policy Diversity Yields Identifiable Representations for Reinforcement Learning
}

\usepackage{authblk}

\author[1,2]{Patrik~Reizinger\thanks{Equal contribution. Correspondence to \href{mailto:patrik.reizinger@tuebingen.mpg.de}{\texttt{patrik.reizinger@tuebingen.mpg.de}}.}\hspace{0.4em}}
\author[3,4]{Bálint~Mucsányi$^*$}
\author[1,5]{Siyuan~Guo$^*$}
\author[6]{Benjamin~Eysenbach}
\author[1,2]{Bernhard~Schölkopf\thanks{Equal supervision.}\hspace{0.4em}}
\author[1,2]{Wieland~Brendel$^\dagger$}

\affil[1]{%
    Max Planck Institute for Intelligent Systems, Tübingen, Germany
}

\affil[2]{%
ELLIS Institute Tübingen, Tübingen, Germany
}
\affil[3]{%
    University of Tübingen, Tübingen, Germany
  }
\affil[4]{%
Tübingen AI Center, Tübingen, Germany
}
\affil[5]{%
University of Cambridge, Cambridge, United Kingdom
}
\affil[6]{%
Princeton University, Department of Computer Science, Princeton, United States
}

\begin{document}

\maketitle

\begin{abstract}
  \input{abstract}
\end{abstract}

\input{main_text}

\bibliographystyle{plainnat}
\bibliography{references,references2}

\newpage
\appendix
\onecolumn
\input{appendix}

\end{document}

%% file: packages.tex
\usepackage{microtype}      
\usepackage{pifont}

\usepackage[hidelinks,backref=page]{hyperref}       

\usepackage{amsfonts}       
\usepackage{amsmath}
\usepackage{amsthm}
\usepackage{amssymb}
\usepackage{thmtools,thm-restate} 
\usepackage{nicefrac}       

\usepackage{svg}
\usepackage{wrapfig}
\usepackage{graphicx}
\usepackage{caption}
\usepackage{subcaption}
\usepackage{float}

\usepackage{tablefootnote}
\usepackage{booktabs}       
\usepackage{multirow}

\let\classAND\AND
\let\AND\relax
\usepackage{algorithmic}

\let\AND\classAND
\AtBeginEnvironment{algorithmic}{\let\AND\algoAND}
\usepackage{algorithm}

\usepackage{awesomebox}
\usepackage{alertmessage}
\usepackage{enumitem}
\usepackage{comment}

\usepackage{xcolor}         
\usepackage{natbib}         
\usepackage{tcolorbox}

%% file: commands.tex
\newcommand{\ie}{i.e.\@\xspace}

\newcommand{\eg}{e.g.\@\xspace}

\newcommand{\cf}{cf.\@\xspace} 

\newcommand{\inv}[1]{\ensuremath{#1^{-1}}}

\newcommand{\transpose}[1]{\ensuremath{#1^{\top}}}

\newcommand{\mat}[1]{\ensuremath{\boldsymbol{\mathrm{#1}}}}

\newcommand{\abs}[1]{\ensuremath{\left|#1\right|}}

\newcommand{\expectation}[1]{\ensuremath{\mathbb{E}_{#1}}}
\newcommand{\rr}[1]{\ensuremath{\mathbb{R}^{#1}}}

\newcommand{\parenthesis}[1]{\ensuremath{\left(#1\right)}}
\newcommand{\brackets}[1]{\ensuremath{\left[#1\right]}}
\newcommand{\braces}[1]{\ensuremath{\left\{#1\right\}}}



%% file: glossary.tex
\usepackage[acronym, automake, toc, nomain, nopostdot, style=tree, nonumberlist,numberedsection]{glossaries}
\usepackage{glossary-mcols}
\usepackage{xparse}
\usepackage{xspace}
\setglossarystyle{mcolindex}

\newglossary{abbrev}{abs}{abo}{Nomenclature}

\newglossaryentry{aux}{
    name        = \ensuremath{\mathrm{\boldsymbol{u}}} ,
    description = {auxiliary variable} ,
    type        = abbrev,
}

\newglossaryentry{im}{
    name        = \ensuremath{\mathrm{Im}} ,
    description = {image space} ,
    type        = abbrev,
}

\newglossaryentry{ker}{
    name        = \ensuremath{\mathrm{Ker}} ,
    description = {kernel space} ,
    type        = abbrev,
}

\newglossaryentry{kronecker}{
    name        = \ensuremath{\otimes} ,
    description = {Kronecker product} ,
    type        = abbrev,
}

\newglossaryentry{loss}{
    name        = \ensuremath{\mathcal{L}} ,
    description = {loss function} ,
    type        = abbrev,
}

\newglossaryentry{numenv}{
    name        = \ensuremath{\abs{E}} ,
    description = {number of environments} ,
    type        = abbrev,
}

\newglossaryentry{lr}{
    name        = \ensuremath{\eta} ,
    description = {learning rate} ,
    type        = abbrev,
}

\newglossaryentry{hypersphere}{
    name        = \ensuremath{\mathcal{S}} ,
    description = {hypersphere} ,
    type        = abbrev,
}
\newcommand{\Sd}[1][d-1]{\ensuremath{\gls{hypersphere}^{#1}}\xspace}

\newglossaryentry{dec}{
    name        = \ensuremath{\boldsymbol{f}} ,
    description = {decoder map $\gls{Latent}\to\gls{Obs}$} ,
    type        = abbrev,
}

\newglossaryentry{deccomp}{
    name        = \ensuremath{f} ,
    description = {decoder map component} ,
    type        = abbrev,
}

\newglossaryentry{enc}{
    name        = \ensuremath{\boldsymbol{g}} ,
    description = {encoder map $\gls{Obs}\to\gls{Latent}$} ,
    type        = abbrev,
}

\newglossaryentry{numdata}{
    name        = \ensuremath{n} ,
    description = {number of samples} ,
    type        = abbrev,
}

\newglossaryentry{observations}{type=abbrev,name=Observations,description={\nopostdesc}}

\newglossaryentry{obs}{
    name        = \ensuremath{\boldsymbol{x}} ,
    description = {observation vector} ,
    type        = abbrev,
    parent      = observations,
}

\newglossaryentry{obscomp}{
    name        = \ensuremath{x} ,
    description = {observation single component} ,
    type        = abbrev,
    parent      = observations,
}

\newglossaryentry{Obs}{
    name        = \ensuremath{\mathcal{X}} ,
    description = {observation space} ,
    type        = abbrev,
    parent      = observations,
}

\newglossaryentry{obsdim}{
    name        = \ensuremath{D} ,
    description = {dimensionality of the observation space \gls{Obs}} ,
    type        = abbrev,
    parent      = observations,
}

\newglossaryentry{obsmat}{
    name        = \ensuremath{\mat{X}} ,
    description = {observation matrix of \rr{\gls{numdata}\times\gls{obsdim}}} ,
    type        = abbrev,
    parent      = observations,
}

\newglossaryentry{obspos}{
    name        = \ensuremath{\tilde{\boldsymbol{x}}} ,
    description = {positive observation vector} ,
    type        = abbrev,
    parent      = observations,
}

\newglossaryentry{obsneg}{
    name        = \ensuremath{{\boldsymbol{x}}^{-}} ,
    description = {negative observation vector} ,
    type        = abbrev,
    parent      = observations,
}

\newglossaryentry{labels}{type=abbrev,name=Labels,description={\nopostdesc}}

\newglossaryentry{label}{
    name        = \ensuremath{\boldsymbol{y}} ,
    description = {label vector} ,
    type        = abbrev,
    parent      = labels,
}

\newglossaryentry{labelhat}{
    name        = \ensuremath{\widehat{\boldsymbol{y}}} ,
    description = {estimated label vector} ,
    type        = abbrev,
    parent      = labels,
}

\newglossaryentry{labelcomp}{
    name        = \ensuremath{y} ,
    description = {label component} ,
    type        = abbrev,
    parent      = labels,
}

\newglossaryentry{labelcomphat}{
    name        = \ensuremath{\widehat{y}} ,
    description = {label component} ,
    type        = abbrev,
    parent      = labels,
}

\newglossaryentry{labelset}{
    name        = \ensuremath{\mathcal{Y}} ,
    description = {label set} ,
    type        = abbrev,
    parent      = labels,
}

\newglossaryentry{labeldim}{
    name        = \ensuremath{C} ,
    description = {number of classes in the label set \gls{labelset}} ,
    type        = abbrev,
    parent      = labels,
}

\newglossaryentry{latents}{type=abbrev,name=Latents,description={\nopostdesc}}

\newglossaryentry{latent}{
    name        = \ensuremath{\boldsymbol{z}} ,
    description = {latent vector} ,
    type        = abbrev,
    parent     = latents,
}

\newglossaryentry{latentcomp}{
    name        = \ensuremath{z} ,
    description = {latent single component} ,
    type        = abbrev,
    parent     = latents,
}

\newglossaryentry{Latent}{
    name        = \ensuremath{\mathcal{Z}} ,
    description = {latents} ,
    type        = abbrev,
    parent     = latents,
}

\newglossaryentry{latentdim}{
    name        = \ensuremath{d} ,
    description = {dimensionality of the latent space \gls{Latent}} ,
    type        = abbrev,
    parent     = latents,
}

\newglossaryentry{latentmat}{
    name        = \ensuremath{\mat{Z}} ,
    description = {latent matrix of \rr{\gls{numdata}\times\gls{latentdim}}} ,
    type        = abbrev,
    parent      = latents,
}

\newglossaryentry{latentpos}{
    name        = \ensuremath{\tilde{\boldsymbol{z}}} ,
    description = {positive latent vector} ,
    type        = abbrev,
    parent      = latents,
}

\newglossaryentry{latentneg}{
    name        = \ensuremath{\boldsymbol{z}^{-}} ,
    description = {negative latent vector} ,
    type        = abbrev,
    parent      = observations,
}

\newglossaryentry{sigmaz}{
    name        = \ensuremath{\boldsymbol{\sigma}_{\gls{latentcomp}}} ,
    description = {std of \gls{latentcomp}} ,
    type        = abbrev,
    parent     = latents,
}

\newglossaryentry{content}{
    name        = \ensuremath{\boldsymbol{z}^{c}} ,
    description = {content latent vector} ,
    type        = abbrev,
    parent     = latents,
}

\newglossaryentry{contentcomp}{
    name        = \ensuremath{z^{c}} ,
    description = {content latent single component} ,
    type        = abbrev,
    parent     = latents,
}

\newglossaryentry{Content}{
    name        = \ensuremath{\mathcal{Z}^{c}} ,
    description = {content} ,
    type        = abbrev,
    parent     = latents,
}

\newglossaryentry{contentdim}{
    name        = \ensuremath{d_{c}} ,
    description = {dimensionality of \gls{content}} ,
    type        = abbrev,
    parent     = latents,
}

\newglossaryentry{sigmac}{
    name        = \ensuremath{\boldsymbol{\sigma}_{c}} ,
    description = {std of \gls{contentcomp}} ,
    type        = abbrev,
    parent     = latents,
}

\newglossaryentry{style}{
    name        = \ensuremath{\boldsymbol{z}^{s}} ,
    description = {style latent vector} ,
    type        = abbrev,
    parent     = latents,
}

\newglossaryentry{stylecomp}{
    name        = \ensuremath{z^{s}} ,
    description = {style latent single component} ,
    type        = abbrev,
    parent     = latents,
}

\newglossaryentry{Style}{
    name        = \ensuremath{\mathcal{Z}^{s}} ,
    description = {style} ,
    type        = abbrev,
    parent     = latents,
}

\newglossaryentry{styledim}{
    name        = \ensuremath{d_{s}} ,
    description = {dimensionality of \gls{style}} ,
    type        = abbrev,
    parent     = latents,
}

\newglossaryentry{sigmas}{
    name        = \ensuremath{\boldsymbol{\sigma}_{s}} ,
    description = {std of \gls{stylecomp}} ,
    type        = abbrev,
    parent     = latents,
}

\newglossaryentry{modality}{
    name        = \ensuremath{\boldsymbol{z}^{m}} ,
    description = {modality-specific latent vector} ,
    type        = abbrev,
    parent     = latents,
}

\newglossaryentry{modalitycomp}{
    name        = \ensuremath{z^{m}} ,
    description = {modality-specific  latent single component} ,
    type        = abbrev,
    parent     = latents,
}

\newglossaryentry{Modality}{
    name        = \ensuremath{\mathcal{Z}^{m}} ,
    description = {latent subspace of \gls{modality}} ,
    type        = abbrev,
    parent     = latents,
}

\newglossaryentry{modalitydim}{
    name        = \ensuremath{d_{m}} ,
    description = {dimensionality of \gls{modality}} ,
    type        = abbrev,
    parent     = latents,
}

\newglossaryentry{algebra}{type=abbrev,name=Algebra,description={\nopostdesc}}

\newglossaryentry{identity}{
    name        = \ensuremath{\boldsymbol{\mathrm{I}}} ,
    description = { identity matrix} ,
    type        = abbrev,
    parent      = algebra,
}

\newglossaryentry{ones}{
    name        = \ensuremath{\boldsymbol{\mathrm{1}}} ,
    description = {a vector of ones} ,
    type        = abbrev,
    parent      = algebra,
}

\newglossaryentry{zeros}{
    name        = \ensuremath{\boldsymbol{\mathrm{0}}} ,
    description = {a vector of zeros} ,
    type        = abbrev,
    parent      = algebra,
}

\newglossaryentry{jacobian}{
    name        = \ensuremath{\boldsymbol{\mathrm{J}}} ,
    description = {Jacobian matrix} ,
    type        = abbrev,
    parent      = algebra,
}

\newglossaryentry{hessian}{
    name        = \ensuremath{\boldsymbol{\mathrm{H}}} ,
    description = {Hessian matrix} ,
    type        = abbrev,
    parent      = algebra,
}

\newglossaryentry{d}{
    name        = \ensuremath{\boldsymbol{\mathrm{D}}} ,
    description = {diagonal matrix} ,
    type        = abbrev,
    parent      = algebra,
}

\newglossaryentry{o}{
    name        = \ensuremath{\boldsymbol{\mathrm{O}}},
    description = {orthogonal matrix} ,
    type        = abbrev,
    parent      = algebra,
}

\newglossaryentry{scalar}{
    name        = \ensuremath{\alpha} ,
    description = {scalar field} ,
    type        = abbrev,
    parent      = algebra,
}

\newglossaryentry{perm}{
    name        = \ensuremath{\mathbb{P}} ,
    description = {group of permutation matrices} ,
    type        = abbrev,
    parent      = algebra,
}

\newglossaryentry{p}{
    name        = \ensuremath{\mat{P}},
    description = {permutation matrix} ,
    type        = abbrev,
    parent      = algebra,
}

\newglossaryentry{prob}{type=abbrev,name=Probability theory,description={\nopostdesc}}

\newglossaryentry{cov}{
    name        = \ensuremath{\boldsymbol{\mathrm{\Sigma}}},
    description = {covariance matrix} ,
    type        = abbrev,
    parent      = prob,
}

\newglossaryentry{mean}{
    name        = \ensuremath{\boldsymbol{\mu}},
    description = {mean} ,
    type        = abbrev,
    parent      = prob,
}

\newglossaryentry{std}{
    name        = \ensuremath{\boldsymbol{\sigma}},
    description = {standard deviation} ,
    type        = abbrev,
    parent      = prob,
}

\newglossaryentry{entropy}{
    name        = \ensuremath{\mathrm{H}} ,
    description = {entropy} ,
    type        = abbrev,
    parent      = prob,
}

\newglossaryentry{expfamparam}{
    name        = \ensuremath{\boldsymbol{\theta}} ,
    description = {parameter of exponential family} ,
    type        = abbrev,
    parent      = prob,
}

\newglossaryentry{expfamnatparam}{
    name        = \ensuremath{\boldsymbol{\eta}} ,
    description = {natural parameter of exponential family} ,
    type        = abbrev,
    parent      = prob,
}

\newglossaryentry{expfamsuffstat}{
    name        = \ensuremath{T(\gls{obs})} ,
    description = {sufficient statistics of exponential family} ,
    type        = abbrev,
    parent      = prob,
}

\newglossaryentry{expfamlogpartition}{
    name        = \ensuremath{A} ,
    description = {log parition function of exponential family (depends on \gls{expfamnatparam})} ,
    type        = abbrev,
    parent      = prob,
}

\newglossaryentry{wishart}{
    name        = \ensuremath{\mathcal{W}} ,
    description = {Wishart distribution} ,
    type        = abbrev,
    parent      = prob,
}

\newglossaryentry{normal}{
    name        = \ensuremath{\mathcal{N}} ,
    description = {normal distribution} ,
    type        = abbrev,
    parent      = prob,
}

\newglossaryentry{matrixnormal}{
    name        = \ensuremath{\mathcal{MN}} ,
    description = {normal distribution} ,
    type        = abbrev,
    parent      = prob,
}

\newglossaryentry{causal}{type=abbrev,name=Causality,description={\nopostdesc}}

\newglossaryentry{cause}{
    name        = \ensuremath{\boldsymbol{N}},
    description = {noise (independent)  variable vector} ,
    type        = abbrev,
    parent      = causal,
}

\newglossaryentry{causecomp}{
    name        = \ensuremath{N},
    description = {noise (independent)  variable component} ,
    type        = abbrev,
    parent      = causal,
}

\newglossaryentry{Cause}{
    name        = \ensuremath{\mathcal{N}} ,
    description = {space of the noise variables} ,
    type        = abbrev,
    parent      = causal,
}

\newglossaryentry{effect}{
    name        = \ensuremath{\boldsymbol{X}},
    description = {observation vector} ,
    type        = abbrev,
    parent      = causal,
}
\newglossaryentry{effectcomp}{
    name        = \ensuremath{X},
    description = {observation component} ,
    type        = abbrev,
    parent      = causal,
}

\newglossaryentry{Effect}{
    name        = \ensuremath{\mathcal{X}} ,
    description = {space of the effect variables} ,
    type        = abbrev,
    parent      = causal,
}

\newglossaryentry{pa}{
    name        = \ensuremath{\boldsymbol{Pa}},
    description = {parents of \gls{effect}} ,
    type        = abbrev,
    parent      = causal,
}

\newglossaryentry{nondesc}{
    name        = \ensuremath{\boldsymbol{ND}},
    description = {non-descendants of \gls{effect}} ,
    type        = abbrev,
    parent      = causal,
}

\newglossaryentry{nondescminuspa}{
    name        = \ensuremath{\boldsymbol{\overline{ND}}},
    description = {non-descendants of \gls{effect}, excluding its parents} ,
    type        = abbrev,
    parent      = causal,
}

\newglossaryentry{semf}{
    name        = \ensuremath{\boldsymbol{f}},
    description = {structural assignment in \glspl{sem}} ,
    type        = abbrev,
    parent      = causal,
}

\newglossaryentry{semfcomp}{
    name        = \ensuremath{f},
    description = {a component of \gls{semf}} ,
    type        = abbrev,
    parent      = causal,
}

\newglossaryentry{order}{
    name        = \ensuremath{\pi},
    description = {causal ordering} ,
    type        = abbrev,
    parent      = causal,
}

\newglossaryentry{indexset}{
    name        = \ensuremath{\mathcal{I}},
    description = {index set} ,
    type        = abbrev,
    parent      = causal,
}
\newglossaryentry{adjacency}{
    name        = \ensuremath{\boldsymbol{\mathcal{A}}} ,
    description = {adjacency matrix of a \glspl{sem}} ,
    type        = abbrev,
    parent      = causal,
}

\newglossaryentry{connectivity}{
    name        = \ensuremath{\boldsymbol{\mathcal{C}}} ,
    description = {connectivity matrix of a \glspl{sem}} ,
    type        = abbrev,
    parent      = causal,
}

\newglossaryentry{dependency}{
    name        = \ensuremath{\mathcal{D}} ,
    description = {dependency matrix of a \glspl{sem}} ,
    type        = abbrev,
    parent      = causal,
}

\newglossaryentry{seq}{
    name        = \ensuremath{\sim_{\acrshort{dag}}} ,
    description = {structural equivalence} ,
    type        = abbrev,
    parent      = causal,
}

\newglossaryentry{contrastive}{type=abbrev,name=Contrastive Learning,description={\nopostdesc}}
\newglossaryentry{clloss}{
    name        = \ensuremath{\mathcal{L}_{\mathrm{\acrshort{cl}}}} ,
    description = {contrastive loss function} ,
    type        = abbrev,
    parent      = contrastive,
}

\newglossaryentry{alignloss}{
    name        = \ensuremath{\mathcal{L}_{\mathrm{align}}} ,
    description = {alignment term in \gls{clloss}} ,
    type        = abbrev,
    parent      = contrastive,
}

\newglossaryentry{uniformloss}{
    name        = \ensuremath{\mathcal{L}_{\mathrm{uniform}}} ,
    description = {uniformity term in \gls{clloss}} ,
    type        = abbrev,
    parent      = contrastive,
}

\newglossaryentry{temp}{
    name        = \ensuremath{{\boldsymbol{\tau}}} ,
    description = {temperature in \gls{clloss}} ,
    type        = abbrev,
    parent      = contrastive,
}

\newglossaryentry{numneg}{
    name        = \ensuremath{M} ,
    description = {number of negative samples} ,
    type        = abbrev,
    parent      = contrastive,
}

\newglossaryentry{vaes}{type=abbrev,name=\acrlongpl{vae},description={\nopostdesc}}

\newglossaryentry{q}{
    name        = \ensuremath{q_{\gls{encpar}}(\gls{latent}|\gls{obs})} ,
    description = {variational posterior of the \acrshort{vae}, mapping $\gls{obs}\mapsto\gls{latent}$ parametrized by \gls{encpar}} ,
    type        = abbrev,
    parent      = vaes,
}
\newglossaryentry{qopt}{
    name        = \ensuremath{q_{\widehat{\gls{encpar}}}(\gls{latent}|\gls{obs})} ,
    description = {optimal variational posterior of the \acrshort{vae}, mapping $\gls{obs}\mapsto\gls{latent}$ parametrized by \gls{encpar}} ,
    type        = abbrev,
    parent      = vaes,
}

\newglossaryentry{encpar}{
    name        = \ensuremath{\boldsymbol{\phi}} ,
    description = {parameters of the variational posterior \gls{q}} ,
    type        = abbrev,
    parent      = vaes,
}

\newglossaryentry{encparopt}{
    name        = \ensuremath{\widehat{\boldsymbol{\phi}}} ,
    description = {optimal parameters of the variational posterior \gls{q}} ,
    type        = abbrev,
    parent      = vaes,
}

\newglossaryentry{var_family}{
    name        = \ensuremath{\mathcal{Q}} ,
    description = {distribution family of the variational posterior \gls{q} } ,
    type        = abbrev,
    parent      = vaes,
}

\newglossaryentry{pz}{
    name        = \ensuremath{p_0(\gls{latent})} ,
    description = {latent prior distribution} ,
    type        = abbrev,
    parent      = vaes,
}

\newglossaryentry{px}{
    name        = \ensuremath{p_{\gls{decpar}}(\gls{obs})} ,
    description = {marginal likelihood } ,
    type        = abbrev,
    parent      = vaes,
}

\newglossaryentry{pdata}{
    name        = \ensuremath{p(\gls{obs})} ,
    description = {data distribution } ,
    type        = abbrev,
    parent      = vaes,
}

\newglossaryentry{mean_enc}{
    name        = \ensuremath{\mu_{\gls{latent}|\gls{obs}}} ,
    description = {mean encoder of the \acrshort{vae}, \ie, $\expectation{\gls{latent}\sim\gls{q}}\parenthesis{\gls{latent}}$, mapping $\gls{obs}\mapsto\gls{latent}$} ,
    type        = abbrev,
    parent      = vaes,
}

\newglossaryentry{var_cov}{
    name        = \ensuremath{\gls{cov}^{\gls{encpar}}_{\gls{latent}|\gls{obs}}} ,
    description = {covariance matrix of \gls{q}} ,
    type        = abbrev,
    parent      = vaes,
}

\newglossaryentry{sigmak}{
    name        = \ensuremath{{\sigma}_{k}^{\gls{encpar}}(\gls{obs})^{2}} ,
    description = {variance of \gls{q} in dimension $k$} ,
    type        = abbrev,
    parent      = vaes,
}

\newglossaryentry{sigmaopt}{
    name        = \ensuremath{\boldsymbol{\sigma}^{\gls{encparopt}}(\gls{obs})^{2}} ,
    description = {optimal variance of \gls{q}} ,
    type        = abbrev,
    parent      = vaes,
}

\newglossaryentry{sigmaoptk}{
    name        = \ensuremath{{\sigma}_{k}^{\gls{encparopt}}(\gls{obs})^{2}} ,
    description = {optimal variance of \gls{q} in dimension $k$} ,
    type        = abbrev,
    parent      = vaes,
}

\newglossaryentry{mu}{
    name        = \ensuremath{\boldsymbol{\mu}^{\gls{encpar}}(\gls{obs})} ,
    description = {mean of \gls{q}} ,
    type        = abbrev,
    parent      = vaes,
}

\newglossaryentry{muk}{
    name        = \ensuremath{{\mu}_{k}^{\gls{encpar}}(\gls{obs})} ,
    description = {mean of \gls{q} in dimension $k$} ,
    type        = abbrev,
    parent      = vaes,
}

\newglossaryentry{muopt}{
    name        = \ensuremath{\boldsymbol{\mu}^{\gls{encparopt}}(\gls{obs})} ,
    description = {optimal mean of \gls{q}} ,
    type        = abbrev,
    parent      = vaes,
}

\newglossaryentry{muoptk}{
    name        = \ensuremath{{\mu}_{k}^{\gls{encparopt}}(\gls{obs})} ,
    description = {optimal mean of \gls{q} in dimension $k$} ,
    type        = abbrev,
    parent      = vaes,
}

\newglossaryentry{gamma}{
    name        = \ensuremath{\gamma} ,
    description = {square root of the precision of the \gls{vae} decoder} ,
    type        = abbrev,
    parent      = vaes,
}

\newglossaryentry{betaloss}{
    name        = \ensuremath{\mathcal{L}_{\beta}} ,
    description = {\betavae loss function} ,
    type        = abbrev,
    parent      = vaes,
}

\newglossaryentry{pxz}{
    name        = \ensuremath{p_{\gls{decpar}}(\gls{obs}|\gls{latent})} ,
    description = {conditional distribution of the decoded samples of the \acrshort{vae}, mapping $\gls{latent}\mapsto\gls{obs}$, parametrized by \gls{decpar}} ,
    type        = abbrev,
    parent      = vaes,
}
\newglossaryentry{pzx}{
    name        = \ensuremath{p_{\gls{decpar}}(\gls{latent}|\gls{obs})} ,
    description = {true posterior distribution of the decoded samples of the \acrshort{vae}, mapping $\gls{obs}\mapsto\gls{latent}$, parametrized by \gls{decpar}} ,
    type        = abbrev,
    parent      = vaes,
}
\newglossaryentry{decpar}{
    name        = \ensuremath{\boldsymbol{\theta}} ,
    description = {parameters of the decoder \gls{pxz}} ,
    type        = abbrev,
    parent      = vaes,
}

\newglossaryentry{invdeccomp}{
    name        = \ensuremath{{g}^{\gls{decpar}}} ,
    description = {inverse decoder component} ,
    type        = abbrev,
    parent      = vaes,
}

\newglossaryentry{invdec}{
    name        = \ensuremath{\mathrm{\boldsymbol{g}}^{\gls{decpar}}} ,
    description = {inverse decoder} ,
    type        = abbrev,
    parent      = vaes,
}

\newglossaryentry{distortion}{
    name        = \ensuremath{D} ,
    description = {Distortion of \cite{alemi_fixing_2018}, the same as the reconstruction term of the \acrshort{elbo} for $\beta=1$} ,
    type        = abbrev,
    parent      = vaes,
}

\newglossaryentry{rate}{
    name        = \ensuremath{R} ,
    description = {Rate of \cite{alemi_fixing_2018}, the same as the \acrshort{kld} term of the \acrshort{elbo} for $\beta=1$} ,
    type        = abbrev,
    parent      = vaes,
}

\newglossaryentry{lindec}{
    name        = \ensuremath{\boldsymbol{\mathrm{W}}} ,
    description = {weight matrix of a linear decoder} ,
    type        = abbrev,
    parent      = vaes,
}

\newglossaryentry{linenc}{
    name        = \ensuremath{\boldsymbol{\mathrm{V}}} ,
    description = {weight matrix of a linear encoder} ,
    type        = abbrev,
    parent      = vaes,
}

\newglossaryentry{imas}{type=abbrev,name=\acrlong{ima},description={\nopostdesc}}

\newglossaryentry{mixing}{
    name        = \ensuremath{\inv{g}} ,
    description = {inverse of the learned unmixing of the \acrshort{ima}, mapping $\gls{latent}\mapsto\gls{obs}$ } ,
    type        = abbrev,
    parent      = imas,
}

\newglossaryentry{lin_mixing}{
    name        = \ensuremath{A} ,
    description = {ground-truth \emph{linear} mixing process of the \acrshort{ima}, mapping $\gls{latent}\mapsto\gls{obs}$ } ,
    type        = abbrev,
    parent      = imas,
}

\newglossaryentry{cima_local}{
    name        = \ensuremath{c_{\acrshort{ima}}} ,
    description = {local \acrshort{ima} contrast } ,
    type        = abbrev,
    parent      = imas,
}

\newglossaryentry{cima_global}{
    name        = \ensuremath{C_{\acrshort{ima}}} ,
    description = {global \acrshort{ima} contrast } ,
    type        = abbrev,
    parent      = imas,
}

\newglossaryentry{source}{
    name        = \ensuremath{s} ,
    description = {sources (\acrshort{ica} equivalent of latents)} ,
    type        = abbrev,
    parent      = imas,
}

\newglossaryentry{rec_s}{
    name        = \ensuremath{\boldsymbol{y}} ,
    description = {reconstructed sources} ,
    type        = abbrev,
    parent      = imas,
}

\newglossaryentry{p_source}{
    name        = \ensuremath{p_{\gls{latent}}} ,
    description = {source distribution} ,
    type        = abbrev,
    parent      = imas,
}

\newglossaryentry{imaloss}{
    name        = \ensuremath{\mathcal{L}_{\gls{ima}}} ,
    description = {\gls{ima} loss function} ,
    type        = abbrev,
    parent      = imas,
}

\NewDocumentCommand{\cima}{ O{\gls{dec}} O{\gls{latent}}  }{\ensuremath{\gls{cima_local} ( #1\!,  #2) }\xspace}
\NewDocumentCommand{\Cima}{ O{\gls{dec}} O{\ensuremath{p_0}  }}{\ensuremath{\gls{cima_global} ( #1,  #2) }\xspace}

\newglossaryentry{gps}{type=abbrev,name=\acrlongpl{gp},description={\nopostdesc}}
\newglossaryentry{gpr}{
    name        = \ensuremath{\mathcal{GP}} ,
    description = {Gaussian Process} ,
    type        = abbrev,
    parent      = gps,
}

\newglossaryentry{gpker}{
    name        = \ensuremath{k} ,
    description = {kernel function} ,
    type        = abbrev,
    parent      = gps,
}

\newglossaryentry{gpcov}{
    name        = \ensuremath{\mathcal{K}} ,
    description = {$\gls{numdata}\times\gls{numdata}$ covariance matrix of a \acrshort{gp}} ,
    type        = abbrev,
    parent      = gps,
}

\makeglossaries

%% file: local_glossary.tex
\newglossaryentry{dim}{
    name        = \ensuremath{d},
    description = {problem dimensionality} ,
    type        = abbrev,
}

\newglossaryentry{r2}{
    name        = \ensuremath{R^2} ,
    description = {coefficient of determination} ,
    type        = abbrev,
}

%% file: config.tex
\definecolor{figblue}{HTML}{4A90E2}
\definecolor{figred}{HTML}{D0021B}
\definecolor{figpurple}{HTML}{7030A0}
\definecolor{figgreen}{HTML}{2CA02C}

\usepackage{cleveref}
\crefname{section}{\S}{\S\S}
\crefname{subsection}{\S}{\S\S}
\crefname{subsubsection}{\S}{\S\S}
\crefname{figure}{Fig.}{Figs.}
\crefname{prop}{Prop.}{Props.}
\crefname{appendix}{Appx.}{Appxs.}
\crefname{algorithm}{Alg.}{Algs.}
\crefname{theorem}{Thm.}{Thms.}
\crefname{definition}{Defn.}{Defns.}
\crefname{cor}{Corollary}{Corollaries}
\crefname{lem}{Lem.}{Lems.}
\crefname{table}{Tab.}{Tabs.}
\crefname{assum}{Assum.}{Assums.}
\crefname{example}{Ex.}{Exs.}

\newtheorem{assum}{Assumption}
\newtheorem{prop}{Proposition}

\newtheorem{lem}{Lemma}

\newtheorem{definition}{Definition}
\newtheorem{example}{Example}
\newtheorem{theorem}{Theorem}

\everypar{\looseness=-1}

\let\ORGhypersetup\hypersetup
\protected\def\hypersetup{\ORGhypersetup}
\setlist[enumerate,itemize]{noitemsep, nolistsep}


%% file: acronyms.tex
\newacronym{mpa}{MPA}{Measure Preserving Automorphism}
\newacronym{iid}{i.i.d.}{independent and identically distributed}
\newacronym{vmf}{vMF}{von Mises-Fisher}
\newacronym{nivmf}{nivMF}{non-isotropic von Mises-Fisher}
\newacronym{pd}{PD}{positive definite}
\newacronym{psd}{PSD}{positive semi-definite}
\newacronym{nd}{ND}{negative definite}
\newacronym{nsd}{NSD}{negative semi-definite}

\newacronym{ode}{ODE}{Ordinary Differential Equation}
\newacronym{pde}{PDE}{Partial Differential Equation}

\newacronym{lhs}{LHS}{left hand side}
\newacronym{rhs}{RHS}{right hand side}

\newacronym{rv}{RV}{random variable}

\newacronym{ae}{AE}{AutoEncoder}
\newacronym{lae}{LAE}{Linear Autoencoder}
\newacronym{vae}{VAE}{Variational Autoencoder}
\newacronym{cvvae}{CV-VAE}{Constant-Variance Variational Autoencoder}
\newacronym{ivae}{iVAE}{Identifiable Variational Autoencoder}
\newacronym{rae}{RAE}{Regularized Autoencoder}
\newacronym{grae}{GRAE}{Gaussian Regularized Autoencoder}
\newacronym{lvm}{LVM}{latent variable model}

\newacronym[longplural=Gaussian Processes]{gp}{GP}{Gaussian Process}
\newacronym{gplvm}{GPLVM}{Gaussian Process Latent Variable Model}
\newacronym{rbf}{RBF}{Radial Basis Function}

\newcommand{\betavae}{$\beta$-\gls{vae}\xspace}

\newacronym{kld}{KL}{Kullback-Leibler Divergence}
\newacronym{elbo}{{\text{\upshape ELBO}}}{evidence lower bound}
\newacronym{pca}{PCA}{Principal Component Analysis}
\newacronym{ppca}{PPCA}{Probabilistic Principal Component Analysis}
\newacronym{ebm}{EBM}{Energy-Based Model}
\newacronym{cca}{CCA}{Canonical Correlation Analysis}

\newacronym{mi}{MI}{Mutual Information}

\newacronym[longplural=Identifiable Exchangeable Mechanisms]{iem}{IEM}{Identifiable Exchangeable Mechanisms}
\newacronym[longplural=Independent Causal Mechanisms]{icm}{ICM}{Independent Causal Mechanisms}
\newacronym{sms}{SMS}{Sparse Mechanism Shift}
\newacronym{mss}{MSS}{Mechanism Shift Score}
\newacronym{sem}{SEM}{Structural Equation Model}
\newacronym{lingam}{LiNGAM}{Linear Non-Gaussian Acyclic Model}
\newacronym{dag}{DAG}{Directed Acyclic Graph}
\newacronym{anm}{ANM}{Additive Noise Model}
\newacronym{cd}{CD}{Causal Discovery}
\newacronym{crl}{CRL}{Causal Representation Learning}
\newacronym{hmm}{HMM}{Hidden Markov Model}
\newacronym{plr}{PLR}{Partially Linear Regression}

\newacronym{it}{IT}{identifiability theory}
\newacronym{sith}{SITh}{Singular Identifiability Theory}

\newacronym{lt}{LT}{learning theory}
\newacronym{slt}{SLT}{Singular Learning Theory}

\newacronym{ica}{ICA}{Independent Component Analysis}
\newacronym{nlica}{NLICA}{nonlinear Independent Component Analysis}
\newacronym{bss}{BSS}{Blind Source Separation}

\newacronym{ima}{{\text{\upshape IMA}}}{Independent Mechanism Analysis}
\newacronym{igci}{IGCI}{Information Geometric Causal Inference}
\newacronym{cdf}{CdF}{Causal de Finetti}

\newacronym{nce}{NCE}{Noise Contrastive Estimation}
\newacronym{pcl}{PCL}{Permutation-Contrastive Learning}
\newacronym{tcl}{TCL}{Time-Contrastive Learning}
\newacronym{gencl}{GCL}{Generalized Contrastive Learning}
\newacronym{iia}{IIA}{Independent Innovation Analysis}

\newacronym{ar}{AR}{autoregressive}
\newacronym{var}{VAR}{Vector autoregressive}
\newacronym{nvar}{NVAR}{Nonlinear Vector AutoRegressive}

\newacronym{ai}{AI}{Artificial Intelligence}
\newacronym{ml}{ML}{Machine Learning}
\newacronym{dml}{DML}{Double Machine Learning}
\newacronym{oml}{OML}{Orthogonal Machine Learning}
\newacronym{homl}{HOML}{Higher-order Orthogonal Machine Learning}
\newacronym{dl}{DL}{Deep Learning}
\newacronym{rl}{RL}{Reinforcement Learning}
\newacronym{mbrl}{MBRL}{Model-Based Reinforcement Learning}
\newacronym{rlhf}{RLHF}{Reinforcement Learning from Human Feedback}
\newacronym{misl}{MISL}{mutual information skill learning}
\newacronym{ssl}{SSL}{self-supervised learning}

\newacronym{cl}{CL}{Contrastive Learning}
\newacronym{dcl}{DCL}{Debiased Contrastive Learning}
\newacronym{scl}{SCL}{Spectral Contrastive Learning}
\newacronym{gcl}{GCL}{Graph Contrastive Learning}
\newacronym{alphacl}{$\alpha$-CL}{$\alpha$-Contrastive Learning}
\newacronym{arcl}{ArCL}{Augmentation-robust Contrastive Learning}
\newacronym{fce}{FCE}{Flow Contrastive Estimation}
\newacronym{pid}{PID}{parametric instance discrimination}
\newacronym{diet}{DIET}{Datum IndEx as its Target}
\newacronym{sdiet}{s-DIET}{Scaled DIET}

\newacronym{vince}{VINCE}{Variational InfoNCE}
\newacronym{rince}{RINCE}{Robust InfoNCE}
\newacronym{aggnce}{AggNCE}{Aggregated InfoNCE}
\newacronym{mcinfonce}{MCInfoNCE}{Monte-Carlo InfoNCE}

\newacronym{gmc}{GMC}{Geometric Multimodal Contrastive Learning}
\newacronym{looc}{LooC}{Leave-one-out Contrastive Learning}
\newacronym{npc}{NPC}{Negative-Positive Coupling}
\newacronym{cpc}{CPC}{Contrastive Predictive Coding}
\newacronym{nlp}{NLP}{Natural Language Processing}
\newacronym{gdl}{GDL}{Geometric Deep Learning}
\newacronym{msn}{MSN}{Masked Siamese Networks}
\newacronym{ifm}{IFM}{Implicit Feature Modification}

\newacronym{dnn}{DNN}{Deep Neural Network}
\newacronym{nn}{NN}{Neural Network}
\newacronym{ann}{ANN}{Artificial Neural Network}

\newacronym{fm}{FM}{Foundation Model}
\newacronym{llm}{LLM}{Large Language Model}
\newacronym{vlm}{VLM}{Vision-Language Model}
\newacronym{pcfg}{PCFG}{Probabilistic Context-Free Grammar}
\newacronym{icl}{ICL}{in-context learning}
\newacronym{ibi}{IBI}{Implicit Bayesian Inference}
\newacronym{cot}{CoT}{Chain-of-Thought}

\newacronym{nc}{NC}{Neural Collapse}
\newacronym{cdt}{CDT}{Class-Dependent Temperature}
\newacronym{mlp}{MLP}{Multi-Layer Perceptron}
\newacronym{fc}{FC}{Fully Connected}
\newacronym{strnn}{StrNN}{Structured Neural Network}

\newacronym{cn}{conv}{Convolutional layer}
\newacronym{cnn}{CNN}{Convolutional Neural Network}
\newacronym{gnn}{GNN}{Graph Neural Network}
\newacronym{ssm}{SSM}{State Space Model}

\newacronym{rnn}{RNN}{Recurrent Neural Network}
\newacronym{lstm}{LSTM}{Long Short-Term Memory}
\newacronym{gru}{GRU}{Gated Recurrent Unit}
\newacronym{relu}{ReLU}{Rectified Linear Unit}
\newacronym{bn}{BN}{Batch Normalization}
\newacronym{dbn}{DBN}{Decorrelated Batch Normalization}

\newacronym{gan}{GAN}{Generative Adversarial Network}

\newacronym{mdp}{MDP}{Markov Decision Process}
\newacronym{pomdp}{POMDP}{partially observable Markov Decision Process}
\newacronym{diayn}{DIAYN}{Diversity Is All You Need}
\newacronym{dads}{DADS}{DYnamics-Aware Discovery of Skills}
\newacronym{visr}{VISR}{Variational Intrinsic Successor Features}
\newacronym{lsd}{LSD}{Lipschitz-constrained Unsupervised Skill Discovery}
\newacronym{csd}{CSD}{Controllability-Aware Unsupervised Skill Discovery}
\newacronym{usd}{USD}{unsupervised skill discovery}
\newacronym{csf}{CSF}{Contrastive Successor Features}
\newacronym{sac}{SAC}{Soft Actor Critic}
\newacronym{a2c}{A2C}{Advantage Actor Critic}

\newacronym{sgd}{SGD}{Stochastic Gradient Descent}
\newacronym{adam}{ADAM}{Adaptive Moment Estimation}
\newacronym{svd}{SVD}{Singular Value Decomposition}
\newacronym{wls}{WLS}{Weighted Least Squares}

\newacronym{sam}{SAM}{Sharpness-Aware Minimization}
\newacronym{samba}{SAMBA}{SAM-Based Autoencoder}
\newacronym{vi}{VI}{Variational Inference}
\newacronym{mfvi}{MFVI}{Mean Field Variational Inference}

\newacronym[longplural=data generating processes]{dgp}{DGP}{data generating process}
\newacronym{map}{MAP}{Maximum A Posteriori}
\newacronym{mle}{MLE}{maximum likelihood estimation}

\newacronym{etf}{ETF}{Equiangular Tight Frame}

\newacronym{mse}{MSE}{Mean Squared Error}
\newacronym{mae}{MAE}{Mean Absolute Error}
\newacronym{ce}{{\text{\upshape CE}}}{cross entropy}

\newacronym{sid}{SID}{Structural Intervention Distance}
\newacronym{shd}{SHD}{Structural Hamming Distance}

\newacronym{mcc}{MCC}{Mean Correlation Coefficient}
\newacronym{mig}{MIG}{Mutual Information Gap}
\newacronym{dci}{DCI}{Disentanglement Completeness Informativeness score}

\newacronym{arc}{ARC}{Average Relative Confusion}
\newacronym{acr}{ACR}{Average Confusion Ratio}

\newacronym{api}{API}{Application Programming Interface}
\newacronym{cpu}{CPU}{Central Processing Unit}
\newacronym{gpu}{GPU}{Graphics Processing Unit}

\newacronym{lti}{LTI}{Linear Time-Invariant}
\newacronym{zoh}{ZOH}{Zero-Order Hold}

\newacronym{gt}{{\text{\upshape GT}}}{ground truth}
\newacronym{ood}{OOD}{out-of-distribution}
\newacronym{oov}{OOV}{out-of-variablme}
\newacronym{fsm}{FSM}{Finite State Machine}
\newacronym{rasp}{RASP}{Restricted-Access Sequence Processing Language}

\newacronym{ntk}{NTK}{Neural Tangent Kernel}

\newacronym{as}{a.s.}{almost surely}
\newacronym{alev}{a.e.}{almost everywhere}

\newacronym{sos}{SOS}{start-of-sequence}
\newacronym{eos}{EOS}{end-of-sequence}

\newacronym{prh}{PRH}{Platonic Representation Hypothesis}
\newacronym{lrh}{LRH}{Linear Representation Hypothesis}

%% file: abstract.tex
    Self-supervised feature learning and pretraining methods in reinforcement learning (RL) often rely on information-theoretic principles, termed mutual information skill learning (MISL). These methods aim to learn a representation of the environment while also incentivizing exploration thereof.
    However, the role of the representation and mutual information parametrization in MISL is not yet well understood theoretically.
    Our work investigates MISL through the lens of identifiable representation learning by focusing on the Contrastive Successor Features (CSF) method. We prove that CSF can provably recover the environment's ground-truth features up to a linear transformation due to
    the inner product parametrization of the features and skill diversity in a discriminative sense.
    This first identifiability guarantee for representation learning in RL also helps explain the implications of different mutual information objectives and the downsides of entropy regularizers.
    We empirically validate our claims in MuJoCo and DeepMind Control and show how CSF provably recovers the ground-truth features both from states and pixels.

%% file: main_text.tex
\section{Introduction}

    The field of \gls{rl} faces several challenges, such as learning under sparse rewards, exploring the environment, and designing an appropriate reward function. Many solutions use different self-supervised approaches including curiosity, intrinsic motivation, or \gls{usd}~\citep{eysenbach_diversity_2018,sharma_dynamics-aware_2020,pathak_curiosity-driven_2017,pathak_self-supervised_nodate,sancaktar_regularity_2023,ha_world_2018}. These methods often use similar information-theoretic arguments to develop self-supervised objectives~\citep{park_metra_2024,park_geometry_2024,eysenbach_diversity_2018,sharma_dynamics-aware_2020,choi_variational_2021}. Some of these objectives are borrowed from \gls{cl}~\citep{eysenbach_contrastive_2022,hansen_fast_2019,park_deep_2021,laskin_unsupervised_2022}. \Gls{misl}~\citep{zheng_can_2024,eysenbach_diversity_2018,achiam2018variational,sharma_dynamics-aware_2020,mohamed_variational_2015, gregor_variational_2017} is a specific class of \gls{usd} methods that uses a \gls{mi} objective.
    \gls{misl} methods have wildly varying performance~\citep{park_metra_2024,zheng_can_2024}, which is theoretically not yet well understood. Empirical evidence suggests guidelines on, \eg, parametrizing the action-value function~\citep{zheng_can_2024,liu_single_2024}, but it does not explain how similar principles lead to such large differences.

    The connection between \gls{ssl} and \gls{misl} methods enables us to investigate this question. We build upon recent theoretical advancements in \gls{ssl}, particularly nonlinear \gls{ica}~\citep{zimmermann_contrastive_2021,hyvarinen_nonlinear_2019,reizinger_cross-entropy_2024,roeder_linear_2020} theory and \gls{crl}~\citep{scholkopf_towards_2021,wendong_causal_2023,reizinger_identifiable_2024,rajendran_interventional_2023}.
    Identifiability results derive guarantees on learning the latent factors from data, \ie, recovering the underlying \gls{dgp} from high-level observations such as pixels. Identifiability results are relevant for \gls{rl} in the partially observable Markov Decision Processes (POMDPs), as the states are not always observed. Thus, the goal is to provably infer the ground-truth states from observations.
    Intuitively, identifiability results aim to learn a representative world model, which they can only achieve when the training samples cover all aspects thereof. Thus, these methods operate under diversity (also called sufficient variability) conditions, which mirror the goal of \gls{usd} to learn diverse skills and aid exploration.

     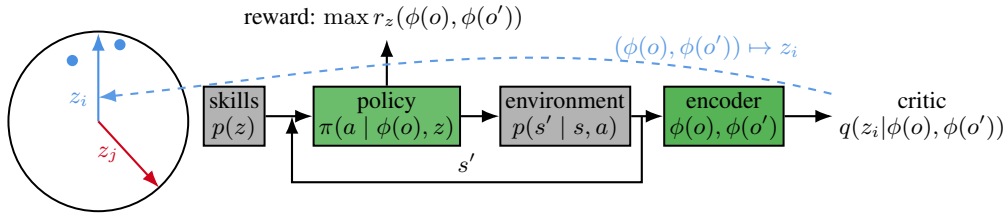
\begin{figure}
        \centering
        \input{figures/fig1.tex}
        \caption{\textbf{The success of \acrfull{misl}~\citep{zheng_can_2024} can be explained by learning identifiable features}:
        The \gls{misl} method \acrfull{csf}~\citep{zheng_can_2024} uses uniformly drawn skills on the hypersphere to learn a skill-conditioned policy by maximizing skill diversity via \eqref{eq:csf_policy}, effectively covering \Sd (spanning \rr{d}). \acrshort{csf} learns an encoder $\phi$ to map observations $o$ to features $\phi(o)$, and uses an inner product parametrization in the critic $q(z_i|\phi(o),\phi(o'))$, where $q$ is trained via a contrastive loss \eqref{eq:csf_cl} to infer the skill $z_i$ from features of consecutive states. Intuitively, this is only possible if each skill is representative of only a subset of states.
        By putting assumptions on the skill-conditioned feature differences $p(s'-s|z)$ and their marginal, we can adapt identifiability results from nonlinear \gls{ica} theory~\citep{reizinger_cross-entropy_2024}. These insights
        show that the above components lead to
        feature identifiability up to a linear transformation (i.e., $\phi(o')-\phi(o) = \mat{A}(s'- s)$ for some matrix $\mat{A}$).
        The skill distribution and the environment dynamics are
        {\color{gray}fixed}, whereas the policy and the encoder are {\color{figgreen}learned}.
        \textbf{Summary:} Our theoretical results help explain the success of prior methods, providing an explanation for why prior work has found some objectives and parametrizations to outperform others
        }
        \vspace{-1em}
        \label{fig:fig1}
    \end{figure}

    \gls{ica} theory shows that estimating \gls{mi} is generally insufficient to learn a ``useful'' representation without further assumptions~\citep{tschannen_mutual_2020,roeder_linear_2020,hyvarinen_nonlinear_1999,locatello_challenging_2019,reizinger_cross-entropy_2024}. Thus, we focus on identifying the distinguishing factors of successful \gls{misl} methods, following~\citep{reizinger_cross-entropy_2024,park_linear_2023}. We show that
    \begin{center}
        \textit{The good performance of \gls{misl} hinges on learning representations via mutual information estimation under diverse policies and an inner product parametrization of the model.}
    \end{center}
    Identifiability also sheds light on the limitations of some methods and design choices, \eg, why the maximum-entropy policy is suboptimal in skill learning, or why feature parametrization matters. By drawing on diversity and variability assumptions from the \gls{ica} and \gls{crl} literatures, we formalize what constitutes a diverse policy and analyze the role of feature dimensionality and skill space coverage, leading to practical insights.
    Our \textbf{contributions} are (\cref{fig:fig1}):
    \begin{itemize}[nolistsep, leftmargin=*]
        \item We explain the success of \acrfull{misl} as the interplay of \gls{mi} estimation under diverse policies and inner product model parametrization, leading to the first identifiability results in \gls{rl} for \acrfull{csf} (\cref{subsec:ident});
        \item Our theoretical results enable us to quantify what a diverse policy means and to pinpoint limitations of previous methods, leading to practical recommendations (\cref{subsec:insights});
        \item We validate our theoretical claims of feature identifiability in \gls{csf} in both state- and pixel-based
        MuJoCo and DeepMind Control environments (\cref{sec:experiments}).
    \end{itemize}

\section{Background}
\label{sec:bg}

    \paragraph{Notation.}
    We consider a \gls{pomdp} without a reward function. We denote successive states as $s$ and $s'$ with initial state distribution $p(s)$, actions as $a$, the state transition distribution as $p(s' \mid s, a)$, observed variables as $o=g(s), g\colon \rr{d}\to \rr{D\geq d}$ being a deterministic generator function, representations or features as $\phi(o)$, and skills as $z_i$, which are all \glspl{rv} sampled from a prior skill distribution $p(z)$. 
    The skill-conditioned policy is $\pi(a \mid o, z)$, and the variational model for the critic is $q(z \mid \phi(o), \phi(o'))$.

    \subsection{Identifiability in \acrlong{ssl}.}

        Self-supervised methods aim to use a pretraining task to learn ``universal'' representations that facilitate solving downstream tasks, such as classification. Many \gls{ssl} methods are related to information-theoretic principles~\citep{zimmermann_contrastive_2021,bizeul_probabilistic_2024,liu_self-supervised_2022,shwartz-ziv_what_2022}. These methods often aim to learn representations such that similar samples have similar representations and \textit{dis}similar samples \textit{dis}similar representations~\citep{wang_understanding_2020}.
        Recent advancements in the nonlinear \gls{ica} literature explained the success of many contrastive \gls{ssl} algorithms by proving their identifiability~\citep{zimmermann_contrastive_2021,rusak_infonce_2024,reizinger_cross-entropy_2024}. Identifiability means that, assuming an underlying \gls{dgp} for the data, the corresponding ``ground-truth'' latent factors can be recovered up to simple transformations (such as linear maps).
        In the nonlinear case, identifiability is only possible with further assumptions~\citep{darmois1951analyse, hyvarinen_nonlinear_1999,locatello_challenging_2019}. A prominent family is that of \textit{auxiliary} variable methods (where the latents are conditionally independent given the auxiliary variable)~\citep{hyvarinen_nonlinear_2019, gresele_incomplete_2019, khemakhem_variational_2020, halva_disentangling_2021,hyvarinen_unsupervised_2016,khemakhem_ice-beem_2020,locatello_weakly-supervised_2020,morioka_connectivity-contrastive_2023,morioka_independent_2021}---as we will show, skills in \gls{misl} can also be interpreted as auxiliary variables. Intuitively, diverse skills are representative of a set of \emph{distinct} states. To prove identifiability, \gls{ica} usually assumes specific \glspl{dgp}, such as energy-based models or inner product parametrization.
        Recently, \citet{reizinger_cross-entropy_2024} showed that the cross-entropy loss is key to explaining why many (self-)supervised deep learning models learn useful (identifiable) representations. We use this insight to prove the identifiability of \gls{csf}.

    \subsection{\Acrfull{misl}}\label{subsec:bg_misl}
        While reinforcement learning (RL) is typically cast as a problem of learning a single policy to maximize a scalar reward function~\citep{sutton1998reinforcement}, recent research focused on self-supervised objectives, inspired by information theory, to incentivize exploration while side-stepping the problem of reward specification.
        The promise of such methods is to learn ``universal'' representations that can be used to solve diverse downstream tasks.
        Paradigms include curiosity~\citep{pathak_curiosity-driven_2017,burda_large-scale_nodate}, regularity~\citep{sancaktar_regularity_2023}, model uncertainty~\citep{treven_optimistic_2023}, and disagreement~\citep{pathak_self-supervised_nodate,sekar_planning_2020,mendonca_alan_2023}, with works showing connections to \gls{cl}~\citep{eysenbach_contrastive_2022,eysenbach_contrastive_2023,zheng_contrastive_2023}.
        Other works in \acrfull{usd} aim to learn a set of policies (i.e., ``skills'') that span the space of behaviors that an agent might perform in a \gls{mdp}~\citep{eysenbach_diversity_2018, achiam2018variational, zheng_can_2024, sharma_dynamics-aware_2020, park2022lipschitz, sutton1998reinforcement}. 
        \gls{misl} is a subset of \gls{usd} methods that relies on information-theoretic principles~\citep{zheng_can_2024,park_metra_2024,sharma_dynamics-aware_2020,eysenbach_diversity_2018}.
        \citet{yang_task_2025} investigates \gls{misl} from a skill learning perspective, though they define the skill to include the parameters of the policy network.

        \paragraph{Representation learning.}
            \gls{misl} is based on representation learning, providing a form of information bottleneck, which is crucial for success~\citep{zheng_can_2024,park_lipschitz-constrained_2021}.
            The learned representation structures the latent space to prescribe a relationship between $(s, s')$ and $z_i.$ Intuitively, it contains information about state transitions and skills, and is usually constrained to the (unit) hypersphere, akin to the common choice in \gls{cl}~\citep{chen_simple_2020}. To reflect the dynamics of the environment, \ie, the relationship between consecutive states $(s,s'),$ it is common to learn feature differences, \ie, $\phi(o')-\phi(o)$ with encoder $\phi$. For a detailed review on representation learning in \gls{rl}, \cf \citet{echchahed_survey_2025}.

        \paragraph{Diverse skills: mixture policies.}
            \gls{misl} aims to learn distinguishable, \ie, \textit{diverse} skills---assuming that solving a variety of tasks requires a versatile policy.
            This is implemented via a skill-conditioned policy, \ie, conditioning on $z_i,$ yielding $\pi(a|o,z)$ which is trained to maximize diversity. Intuitively, the marginal policy $\pi(a|s) = \int \pi(a|o,z)p(z)$ can be thought of as a mixture policy, where each mixture component is a different tool in the agent's toolbox (\cf \cref{ex:intuition}).
            Diversity means that a discriminative model can uniquely infer the skill $z_i$ from consecutive states $(s,s')$. This notion relates to sufficient variability conditions in \gls{ica}~\citep{hyvarinen_unsupervised_2016,hyvarinen_nonlinear_2019,khemakhem_ice-beem_2020}
            and interventional discrepancy from \gls{crl}~\citep{wendong_causal_2023}.
            Formally:
       \begin{definition}[Diverse skill-conditioned policies]\label{def:diverse_policy}
            We call a skill-conditioned policy $\pi(a|o,z)$ diverse if an ideal discriminative model can uniquely infer the skill from consecutive states $(s,s')$. Alternatively, if for given state transitions $p(s'|s, a)$, the integral $\int p(s'|s,a)\pi(a|o,z_i)p(s|z_i)ds$ differs almost surely for any $z_i, z_{j\neq i}.$
        \end{definition}

        \paragraph{Architecture: inner product parametrization.}
            \gls{misl} objectives approximate \gls{mi} via an \gls{elbo}, for which they require a variational approximation. The variational posterior (\ie, the critic or the $Q$-value function) is often parametrized as an inner product, which is critical for achieving great performance~\citep{zheng_can_2024}, though an explanation is yet to be provided.
            This parametrization is prevalent in \gls{ssl}, and was theoretically shown to be crucial for identifiability guarantees~\citep{roeder_linear_2020,zimmermann_contrastive_2021,hyvarinen_nonlinear_2019,hyvarinen_unsupervised_2016,khemakhem_ice-beem_2020,reizinger_cross-entropy_2024}.

     \paragraph{\Acrfull{csf}~\citep{zheng_can_2024}.}
        Our identifiability proof in \cref{sec:theory} is for a prototypical \gls{misl} method, \gls{csf}~\citep{zheng_can_2024}, which is representative of prior work~\citep{park_metra_2024, gregor_variational_2017, warde2018unsupervised}.
        \gls{csf} learns both a feature representation via an encoder and a skill-conditioned policy; the skills are represented by vectors $z$ drawn uniformly from the hypersphere (\cf \cref{fig:fig1}).\\
        \textbf{State representation.}
        \gls{csf} learns a probabilistic critic $q(z|\phi(o), \phi(o'))$ with encoder $\phi$ to discriminate the skills based on consecutive  observations $(o,o')$ corresponding to the state transition $(s,s')$. The encoder can be trained either from direct state observations or from pixels. The loss is a contrastive lower bound on the mutual information $I(s,s';z)$:
         \begin{equation}
            q(z_i|\phi(o), \phi(o')) = \dfrac{p(z_i) \exp\brackets{\transpose{\parenthesis{\phi(o')-{\phi(o)}}}z_i}}{\expectation{p(z)}\exp\brackets{\transpose{\parenthesis{\phi(o')-\phi(o)}}z}},\label{eq:csf_cl}
        \end{equation}
        which is equivalent to a cross-entropy loss, as it was shown for different parametrizations~\citep{hyvarinen_unsupervised_2016,hyvarinen_nonlinear_2019,zimmermann_contrastive_2021,rusak_infonce_2024}.
        This loss can be equivalently seen as a parametric instance discrimination objective~\citep{wu_unsupervised_2018, oord_representation_2019,he2020momentum} on the feature difference \parenthesis{\phi(o')-{\phi(o)}}, akin to formulations in prior work~\citep{ibrahim_occams_2024,reizinger_cross-entropy_2024}.
         Prior work has shown that parametrizing the critic as a log-linear model with an
        inner product parametrization is crucial for identifiability
    \citep{hyvarinen_nonlinear_2019,roeder_linear_2020,zimmermann_contrastive_2021,reizinger_cross-entropy_2024}.\\
            \textbf{The policy.}
            The skill-conditioned policy $\pi(a \mid \phi(o), z)$ is learned by using RL to optimize the reward function $r_z(\phi(o), \phi(o')) = (\phi(o') - \phi(o))^\top z$, where $\phi$ is the same encoder as above:
                \begin{align}
                    &\pi \!=\! \arg\max_{\pi} \mathbb{E}_{p(z)}\expectation{\pi(\cdot \mid \cdot, z)}\hspace{-0.3em}\left[ \sum_{t=0}^\infty \gamma^t r_z(\phi(o), \phi(x'))\!\right] \!\!=\! \mathbb{E}_{p(z)} \hspace{-0.3em}\left[\!\expectation{\pi(\cdot \mid \cdot, z)}\hspace{-0.3em}\!\left[ \sum_{t=0}^\infty \gamma^t (\phi(o') \!-\! \phi(o))\!\right]^\top \hspace{-0.3em}z\right]\!\!, \\
                    & \text{where} \quad s' \sim p(s' \mid s,a), \quad
                    a \sim \pi(a \mid o,z), \quad
                    p(z) \stackrel{d}{=} \text{Uniform}(\Sd[d-1]). \label{eq:csf_policy}
                \end{align}

    \subsection{Connections between \gls{rl}, \gls{ssl}, and \gls{crl}.}
        The notion of agency is central in \gls{rl}, but it is not unique to it. When the causality literature~\citep{pearl_causality_2009,spirtes2000causation} reasons about interventions, it (implicitly) assumes that an agent could change the environment. Causality describes cause-effect relationships with \glspl*{sem}~\citep{pearl_causality_2009,spirtes2000causation}. To identify causal relationships, interventional data is often required.
        \gls{crl} methods are often tested on tasks involving an agent such as a robotic arm or a control system~\citep{scholkopf_towards_2021,liu_causal_2023,lippe_citris_2022,lippe_icitris_2022,yang_causalvae_2023}.
        \citet{rajendran_interventional_2023} showed how to design interventions (\ie, the policy) for control systems to provably identify the world model. This was an important conceptual step to connect the two fields, as conventional identifiability theory assumes access to pre-collected data.
        The problem of reinforcement learning is explicitly about collecting data (i.e., exploration); RL methods (such as \gls{misl}) based on self-supervised learning (including \gls{cl}) provide another missing link
        \citep{eysenbach_contrastive_2022,eysenbach_contrastive_2023,zheng_contrastive_2023,zimmermann_contrastive_2021}, especially given that there are close connections between \gls{cl}, nonlinear \gls{ica}, and identifiability~\citep{zimmermann_contrastive_2021,rusak_infonce_2024}.

        \paragraph{Modeling the \gls{misl} problem as a \acrfull{dgp}.}
        Identifiability guarantees require assumptions on the \gls{dgp}, \ie, we need to connect the \gls{pomdp} of the \gls{rl} problem to a \gls{dgp}.
        This corresponds to positing a probabilistic model of states, actions, transitions, and skills.
        To match the \gls{ica} literature, we model the skills as a set, with each skill having a corresponding high-dimensional unit vector $z_i\in\Sd$.
        This is slightly different from how \gls{misl} methods handle the skills by sampling them for each rollout from, \eg, from a uniform $p(z).$
        However, as identifiability guarantees only require that the skills span $\rr{d},$ these modeling choices are compatible, as with enough samples, the skill vectors space $\rr{d}$ almost surely.
        The \gls{pomdp} includes the model of the state transitions $p(s'|s,a),$ from which we observe $a$, but might not directly observe $s$, only a function thereof via the generator $o=g(s).$ The goal of \gls{ica} is to invert $g,$ \ie, to extract the state $s$ from $o$.
        ICA further requires making assumptions about the probabilistic model, typically about a conditional distribution. As \gls{misl} methods aim to infer the skill from state pairs, we will assume a specific form for the conditional $p(s'-s|z),$ where the feature difference $s'-s$ is defined in $\Sd.$

\section{Identifiability insights in \acrfull{misl}}\label{sec:theory}

    \paragraph{Motivation.}
        Recent works have shown fundamental connections between \gls{rl}, \gls{ssl}, and causality, pointing out that the borders between these fields are not as crisp as one might have expected. For example, there is a long line of work using self-supervised learning to drive \gls{rl}, particularly in skill discovery~\citep{eysenbach_diversity_2018,eysenbach_contrastive_2022,sharma_dynamics-aware_2020,pathak_curiosity-driven_2017,pathak_self-supervised_nodate,hansen_fast_2019,park_metra_2024,park_controllability-aware_2023}. Indeed, the \gls{rl} problem is fundamentally about agents taking interventions (i.e., actions).
        \citet{reizinger_jacobian-based_2023} proved that identifiable representation learning can answer causal queries by extracting the causal \gls{dag} without the explicit notions of agency or interventions. \citet{rajendran_interventional_2023} demonstrated how intervention design in a simple control system can provide data satisfying the assumption for \gls{ica}. Most recently, \citet{reizinger_identifiable_2024} proposed a unifying framework encompassing interventional and multi-environmental (\ie, when distribution shifts occur) identifiability results for causal structure and representations with the statistical notion of exchangeability.
        In the same way that identifiability theory helps explain the success of self-supervised pretraining in computer vision~\citep{zimmermann_contrastive_2021,rusak_infonce_2024,ibrahim_occams_2024,reizinger_cross-entropy_2024}, we will use recent insights from identifiability to study self-supervised RL methods. Using
        \citet{zheng_can_2024} as a prototypical \gls{misl} method, we will provide theory to elucidate why and when \gls{misl} works by proving that \gls{csf} identifies the ground-truth states of the underlying \gls{mdp} (\cref{subsec:ident}). Our analysis will also use \gls{ica} theory to reason about failure cases and to formulate practical recommendations (\cref{subsec:insights}).

    \paragraph{Intuition.}
        \gls{rl} and representation learning methods are often pretrained via self-supervised tasks to learn a ``universal'' representation that can solve many downstream tasks (\cref{subsec:bg_misl}).
        \begin{center}
            {\textit{Our insight is that learning diverse skills is equivalent to learning to distinguish data under different distribution shifts or interventions, and this leads to \gls{rl} agents identifying the ground-truth states of the underlying \gls{pomdp} up to a linear transformation.}}
        \end{center}
        Before analyzing the relevant technical assumptions, we provide an illustrative example of how distinguishing skills can be useful to learn the states of the underlying \gls{pomdp}:
        \begin{example}\label{ex:intuition}
            Assume that a robot moves around in a maze to create a map of it. However, it does not have access to other sensory information but a camera. To create the map, \ie, to learn the underlying state information such as the position of the robot, the walls, or other objects, it needs to move around to collect representative images.
            The ICA setting assumes that we already have such images and aims to reconstruct the state.
            Skill-based \gls{rl} solves a harder problem, as it also needs to learn a policy to explore while learning the underlying state representations. The question we answer in this work is: do those representations identify important quantities such as position and orientation?
        \end{example}

    \subsection{The Identifiability of \acrfull{csf}}\label{subsec:ident}

        Our key insight is to analyze recent advances in \gls{misl}~\citep{eysenbach_diversity_2018,hansen_fast_2019,sharma_dynamics-aware_2020,park_lipschitz-constrained_2021,eysenbach_contrastive_2022,park_controllability-aware_2023,park_metra_2024,zheng_can_2024} through the lens of  identifiability theory:
        \vspace{-3pt}
        \begin{center}
            \textit{The success of \gls{misl} methods, including the role of diversity and importance of a linearly parametrized critic, can be explained by identifiability theory.}
        \end{center}
        \vspace{-3pt}
        To show the identifiability of the \gls{csf} features, for this section, we assume that we have access to a diverse skill-conditioned policy and recall how we related the \gls{pomdp} to a \gls{dgp}. We proceed in the following steps:
        \begin{enumerate}[label=(\roman*)]
            \item We show that given a diverse policy, the collected data (state trajectories or observations thereof) satisfy the assumptions required for identifiability in \gls{ica}.
            \item Then we investigate what this identifiability result implies for \gls{csf}.
        \end{enumerate}

        \paragraph{Matching the assumptions of \gls{csf} to \gls{ica}.}
        We start by investigating whether the assumptions of nonlinear \gls{ica} theory match those in \gls{misl}.
        When the following assumptions are fulfilled, the features are identified up to a linear transformation, \ie, one can fit a linear map between the features learned by the model and the true ones (\eg, if one has access to them, such as in a simulator).
        \begin{assum}[Adapted from~\citep{reizinger_cross-entropy_2024} (Assm.~1C)]\label{assum:rl_ident}
            We assume that for consecutive states $s, s' \in \rr{d}$,  $(s'-s)\in\Sd$ and consider skills $z_i\in \Sd$ that satisfy the following.
            \begin{enumerate}[label=(\roman*), leftmargin=*]
                \item The finite set of skills is unit-normalized and forms an affine generator system of \rr{d} (\cref{def:affine_generator}).
                \item The skill-conditioned features $p(s'-s|z)$ follow a \gls{vmf} distribution on the hypersphere with mean $z$:
                \vspace{-6pt}
                \begin{equation}
                     (s'-s) \sim p(s'-s|z) \propto e^{\kappa \langle z, s'-s \rangle}. \label{eq:vmf}
                     \vspace{-6pt}
                \end{equation}
                \vspace{-3pt}
                \item The marginal of the features $(s'-s)$ is uniform  on the hypersphere.
                \item Each pair of features corresponds to one skill. That is, an ideal discriminator can uniquely map $(s'-s) \mapsto z$, yielding diverse skills (\cf \cref{def:diverse_policy}).
                \item The critic $q(z|\phi(o),\phi(o'))$ uses an encoder $\phi:\rr{D}\to\rr{d}$ to learn the features with an inner product parametrization $\transpose{\brackets{\phi(o')-\phi(o)}}z_i$, it optimizes a contrastive objective~\eqref{eq:csf_cl}, and is expressive enough to reach the global optimum of the objective.
                \item The observations $o$ are generated by passing the latent state $s$ through a continuous and injective generator function $g\colon \Sd\! \rightarrow \rr{D}$, \ie, $o = g(s)$, where $D\geq d$.
            \end{enumerate}
        \end{assum}

        \begin{proof}[Assumption feasibility]
            As we are modeling a practical scenario, we need to investigate whether these assumptions are realistic:
            \begin{enumerate}[label=(\roman*), leftmargin=*]
                \item In \gls{csf} skills are drawn uniformly from the hypersphere, and if we have sufficiently many of them, they span \rr{d} almost surely, thus they almost surely form an affine generator system. We demonstrate that this setup is sufficient, but not necessary: a set of discrete skills also leads to high identifiability scores (\cref{fig:skill_diversity}).
                \item Empirical observations show that the learned  features $\phi(o')-\phi(o)$ follow a \gls{vmf} for \gls{csf}~\citep[Fig.~2(a-b)]{zheng_can_2024}
                \item Empirical observations  show the learned  features $\phi(o')-\phi(o)$ follow a uniform marginal on the hypersphere for \gls{csf}~\citep[Fig.~2(c)]{zheng_can_2024}
                \item The policy optimizes $\transpose{\brackets{\phi(o')-\phi(o)}}z_i$. As we show in \cref{lem:csf_diversity}, a maximum entropy policy is not diverse. Applying our argument to pairs of skills, one can always improve diversity if the state transitions depend on any one of those two skills (as opposed to no skill dependence, which is implied by a maximum entropy policy), which implies that each state transition can be mapped to a single skill. Refer to \cref{lem:csf_diversity} for details.
                \item The neural networks used for training use an inner product parametrization, and empirical evidence by \citep{zheng_can_2024} showed that this works the best in practice.
                \item In practice, MISL methods are either trained from directly observing the state (\ie, $g$ is the identity) or from pixels. Both cases reasonably fulfill the assumption.
            \end{enumerate}
            \vspace{-.6em}
        \end{proof}
        \vspace{-12pt}

    \paragraph{Identifiability of the underlying features.}
    Under \cref{assum:rl_ident}, the feature differences are identified up to a linear map, as shown by  \citet{reizinger_cross-entropy_2024}:
       \begin{theorem}[Self-Supervised Learning Identifiability~\citep{reizinger_cross-entropy_2024} (Thm.~1)]
        \label{thm:ident_theo_supervised}
            When \cref{assum:rl_ident} holds and a continuous encoder \(\phi\colon \mathbb{R}^D \rightarrow \mathbb{R}^d\) and a linear classifier \(Z\) globally minimize the cross-entropy objective, then the composition \(h = \phi \circ g\) is a linear map from \(\mathbb{S}^{d-1}\) to \(\mathbb{R}^d\), \ie, ${\phi(o')-\phi(o)}
                = \mat{A}\brackets{s'-s}$ where $\mat{A}\in\rr{d\times d}$ is a linear map.
        \end{theorem}
        \cref{thm:ident_theo_supervised} means that the features learned by \gls{misl}~\citep{zheng_can_2024} will correspond to the ground-truth states of the underlying \gls{pomdp} up to a linear transformation.
        As the linear map in \cref{thm:ident_theo_supervised} is the same for all states, this, combined with the linear parametrization, implies that the feature differences are also identified up to a linear transformation:
        \begin{prop}[Feature identifiability in \gls{csf}]\label{cor:ident_successor}
            \cref{thm:ident_theo_supervised} implies by the inner product parametrization of $q(z|s,s')$ that $\phi(o) = \mat{A}s$ and $\phi(o') = \mat{A}s'$ with the same \mat{A}, thus, the features are also identified up to a linear transformation.
        \end{prop}
        \vspace{-10pt}
        \begin{proof}
            The proof follows from the linear parametrization of the model. By linear identifiability of the feature differences (as they lie on \Sd, an offset is not possible), we have
            \vspace{-6pt}
            \begin{align*}
               {\phi(o')-\phi(o)} &= \mat{A}\brackets{s'-s} = \brackets{\mat{A}s'-\mat{A}s}
               \vspace{-2.8\baselineskip}
            \end{align*}
            \vspace{-.85em}
        \end{proof}
        Having identifiability both for $s$ and $(s'-s)$ might suggest that it does not matter whether the objective optimizes a lower bound on $I(s,s';z)$ or $I(s;z).$ As we show in \cref{subsec:insights}, the difference lies in the additional geometric constraints on the latent space. Namely, there exist spurious solutions of the InfoNCE objective that do not preserve the structure of the latent space~\citep{wang_chaos_2022}.
        \vspace{-3pt}

    \subsection{Insights from ICA theory}\label{subsec:insights}

     \begin{figure}
        \centering
        \input{figures/mi_insights}
        \caption{\textbf{Implications of optimizing different mutual information objectives $I(s, s'; z_i)$ vs. $I(s; z_i)$}: the reward function and the feature learning objective impose different inductive biases on the structure of the latent space. \textbf{Left:} using feature $\phi(s')-\phi(o)$ to parametrize $I(s, s'; z_i)$ ensures that the embeddings of consecutive states are close but distinct. \textbf{Right:} optimizing $I(s; z)$ does not impose a ``locality'' constraint on the embeddings of consecutive states while incentivizing that both should be parallel to skill $z_i$. This either results in collapsed (\ie, $\phi(o)=\phi(o')$) or antipodal features.
        \vspace{-1.25\baselineskip}
        }
        \label{fig:mi_insights}
    \end{figure}
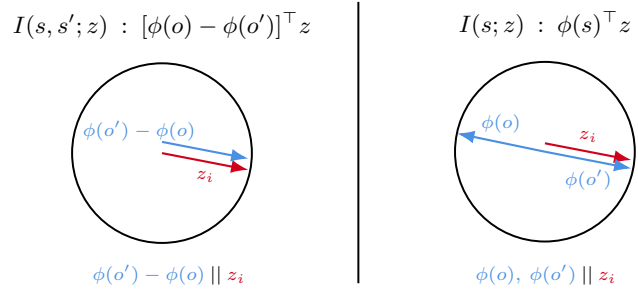

        \paragraph{Mutual information formulation matters for the geometry of feature space.}
            In the literature, there are many choices for optimizing \gls{mi}, namely, $I(s, s'; z)$ versus $I(s_0, s; z)$ versus $I(s; z)$---\cf \cref{fig:mi_insights} for a comparison.
            By looking into the policy and analyzing what maximizing $\expectation{s,z,a}\transpose{\brackets{\phi(o') - \phi(o)}}z$ means, we hope to shed light on the advantages of  $I(s, s'; z)$ over the latter versions.
            Maximizing the inner product $\transpose{\brackets{\phi(o') - \phi(o)}}z$ means that the difference $\phi(o') - \phi(o)$ needs to be parallel to $z$. This implies that neither $\phi(o)$ nor $\phi(o')$ can collapse to the same vector, as they need to be distinct such that their difference is parallel to $z$.
            On the other hand, optimizing $I(s; z)$ would mean two separate conditions for $\phi(o)$ and $\phi(o')$. But if both are parallel to $z$, then they are either parallel or antipodal. Importantly, neither are realistic modeling assumptions, as we assume that consecutive states should be close to each other in embedding space (but neither the same, nor very far apart, \cf \cref{assum:rl_ident}(ii)).
            The same argument holds for $I(s_0, s; z)$ with the difference of offsetting the whole space by the initial state. If the differences of $s-s_0$ and $s'-s_0$ are both parallel to $z$, then they are either parallel or antipodal.
            Instead, $I(s, s'; z)$ prescribes the ``closeness'' of consecutive states---though $\phi(o')-\phi(o)$ could in this case also be antipodal to $z_i$, \citep[Fig.~2b]{zheng_can_2024} suggests that this is not the case.

        \paragraph{A practical implication of diversity rewards.}
            A perhaps interesting interpretation of rewards such as \eqref{eq:csf_policy} is that it quantifies a notion of data diversity. From a theoretical perspective, diversity is a binary question, as it is required to make a matrix invertible. But this matrix can also be ill-conditioned but not rank-deficient, leading to performance deterioration~\citep{rajendran_interventional_2023}. Understanding when the reward is a good predictor of learning useful representations from a given data set  (\eg, in offline RL) is an interesting avenue for future work.

        \paragraph{Maximum entropy policies lead to worse performance.}
            Perhaps surprisingly, entropy is not a suitable quantity for skill diversity, as a maximum entropy policy breaks the dependence on the skill of the skill-conditioned policy:

            \begin{lem}[A maximum-entropy policy in CSF is not diverse]\label{lem:csf_diversity}
                A maximum entropy skill-conditioned policy $\pi(a|o,z)= \text{Uniform}$ is not diverse and cannot maximize the reward $\expectation{s,z,a}\transpose{\brackets{\phi(o') - \phi(o)}}z$.
            \end{lem}
            \begin{proof}[Indirect]
                Fix the initial state and assume that the skill-conditioned policy has maximum entropy, \ie, it follows a uniform distribution and maximizes the reward---given that the policy network is sufficiently flexible to express such a policy.  This implies that the expectation over the actions does not depend on $z$, yielding in expectation the same $\brackets{\phi(o') - \phi(o)}$ for each skill.
                \vspace{-6pt}
                \begin{align}
                    r_z(\phi(o),\phi(o'))&= \expectation{s,z,a}\transpose{\brackets{\phi(o') - \phi(o)}}z\\
                    &= \int_{s,z,a}\transpose{\brackets{\phi(o') - \phi(o)}}zp(s'|s,a)\pi(a|o,z)p(z)p(s)dsdzda
                    \vspace{-6pt}
                    \intertext{In this case, the skill-conditioned policy becomes independent of $z_i$, as the uniform distribution over the action space has maximum entropy. Substituting $\pi(a|o,z)=\pi(a|o)$ yields}
                    &= \int_{s,z,a}\transpose{\brackets{\phi(o') - \phi(o)}}zp(s'|s,a)\pi(a|o)p(z)p(s)dsdzda
                    \vspace{-8pt}
                    \intertext{and by reordering the terms, we get}
                    \vspace{-6pt}
                    &= \int_{s,z,a}\transpose{\brackets{\phi(o') - \phi(o)}}z\pi(a|o)p(s'|s,a)p(s)dsdap(z)dz.
                    \vspace{-6pt}
                \end{align}
                Note that $p(s'|s,a)$, $\pi(a|o)$, and $p(s)$ are independent of $z_i,$ thus, \brackets{\phi(o') - \phi(o)} are also independent of $z_i$. As a skill vector parallel to \brackets{\phi(o') - \phi(o)} maximizes the inner product, and the skills are on the unit hypersphere, this yields a unique solution. However, as the skills are drawn uniformly from the unit hypersphere, they are distinct. That is, the inner products will differ for $z_i \neq z_{j\neq i}$. Thus, both skills cannot maximize the reward, leading to a contradiction.
            \end{proof}

                Intuitively, if the actions---and, thus, the state transitions---do not depend on the skill, then from a given state pair $(s,s')$ it is impossible to infer the skill with the discriminative model $q(z|s,s'),$ and the reward cannot be optimal.
                 \begin{figure}
                    \centering
                    \includegraphics{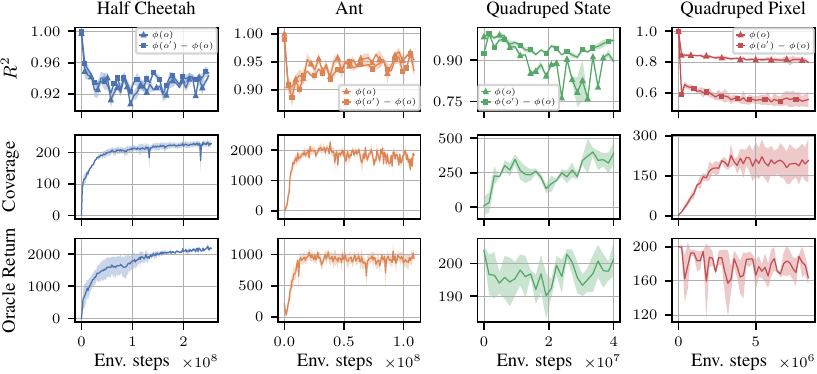}
                    \caption{\textbf{\gls{csf} identifies the underlying states in MuJoCo and DMC up to a linear transformation.}  \textbf{Top:} Identifiability of both features $\phi(o)$ and feature differences $\phi(o')-\phi(o)$, measured by the \gls{r2} score (higher is better); \textbf{Middle:} state coverage, indicating exploratory behavior; \textbf{Bottom:} oracle return indicating zero-shot task transfer performance. Error bars represent one standard deviation.
                    \vspace{-\baselineskip}}
                    \label{fig:results_mujoco}
                \end{figure}

\section{Experiments}
\label{sec:experiments}

    \paragraph{Setup.}
        We use the codebase of~\citet{zheng_can_2024} and run experiments in the MuJoCo and DeepMind Control (DMC) environments with the \gls{csf} algorithm. During self-supervised pretraining, we monitor the learned (successor) features. To evaluate identifiability, we set the feature dimensionality to match the number of ground-truth features.
        We use either states (Half Cheetah, Ant, Quadruped State) or pixels (Quadruped Pixel) as observations to learn the features $\phi(o)$. In the state-based environments, the encoder is, in principle, able to represent the identity map (leading to a perfect $R^2$ score); however, it is unclear whether such features would optimize the loss. Unless otherwise noted, error bars represent two standard deviations. Refer to \cref{sec:app_exp} for more details.

    \paragraph{Metrics.}
    The benefit of using simulated MuJoCo and DMC environments is having access to the ground truth states, enabling one to evaluate the relationship between the learned features $\phi(o)$ and the ground truth latent factors $s$.
    As the identifiability guarantees discussed in~\cref{thm:ident_theo_supervised,cor:ident_successor} hold up to a linear transformation, we fit a linear map \mat{A} between the features $\phi(o)$ and the ground truth states $s$ by minimizing $\Vert s - \mat{A}\phi(o)\Vert_2^2$ and report the coefficient of determination \gls*{r2}~\citep{wright1921correlation} of the linear fit.
    To connect downstream performance with state identifiability, we also report (i) the \textbf{state coverage} of agents using skills from $p(z)$ and (ii) the \textbf{oracle reward} from the best rollout among the sampled skills. The state coverage measures the effectiveness of exploration independent of any environment reward. The oracle reward evaluates \emph{zero-shot} skill transfer: a diverse skill set should naturally contain one that performs well on the downstream task. Importantly, we do not evaluate \emph{supervised} skill transfer, as it can lead to a collapse in the latent representations by learning shortcuts, akin to in computer vision~\citep{geirhos_shortcut_2020}.

    \paragraph{Results.} \cref{fig:results_mujoco} evaluates the identifiability of CSF across three state-based MDPs (Cols.~1-3) and one pixel-based environment (Col. 4). CSF shows \emph{strong identifiability} in the state-based environments for both features $\phi(o)$ and their differences $\phi(o') - \phi(o)$, aligning with~\cref{thm:ident_theo_supervised,cor:ident_successor}. In the state-based environments, it is by no means obvious that the encoder will keep all information about the underlying states, \eg, there could have existed a shortcut solution that optimizes the \gls{csf} objective while discarding some information about the states. In the pixel-based Quadruped environment,
    the individual features $\phi(o)$ still have a high $R^2$ score with the underlying states, while the feature differences exhibit a more moderate linear relationship.

    \begin{figure}
        \centering
        \begin{minipage}{0.48\textwidth}
            \centering
            \includegraphics[width=\linewidth]{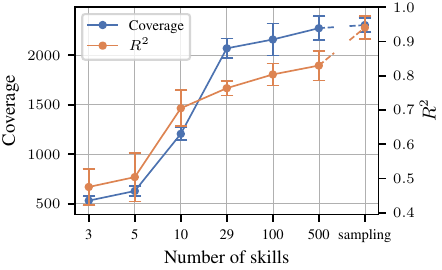}
            \caption{\textbf{The effect of skill diversity on state identifiability and coverage in the Ant environment.} Skills are sampled from $p(z)$ at the start of pretraining and kept fixed throughout, except in the `sampling' case where skills are redrawn from $p(z)$ during training, emulating an infinite set of skills. An insufficient number of skills violates \cref{assum:rl_ident}(i), leading to both a lower state coverage and \gls{r2} score.}
            \label{fig:skill_diversity}
        \end{minipage}
        \hfill
        \begin{minipage}{0.48\textwidth}
            \centering
            \includegraphics[width=\linewidth]{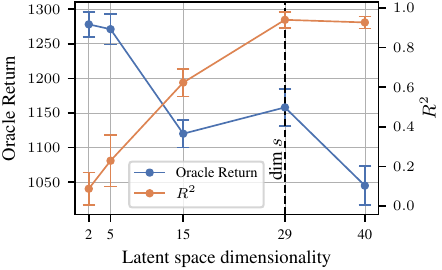}
            \caption{\textbf{The effect of latent space dimensionality on state identifiability and zero-shot task transfer in the Ant environment.} Linear identifiability requires that the feature space has at least as many dimensions as the true state (\cf \cref{assum:rl_ident}(v)).
            In contrast, a smaller latent space aids task transfer but has a non-monotonic relationship with dimensionality, with the state dimensionality being a local optimum.}
            \label{fig:latent_dim}
        \end{minipage}
        \vspace{-\baselineskip}
    \end{figure}

    Inspired by the diversity conditions in nonlinear \gls{ica}~\citep{hyvarinen_nonlinear_2019,rajendran_interventional_2023,wendong_causal_2023,reizinger_identifiable_2024},  we investigate the effect of skill diversity on state identifiability and coverage. Namely, \cref{thm:ident_theo_supervised,cor:ident_successor} only require that the skills span $\rr{d},$ suggesting that a limited set of skills might be sufficient to achieve identifiability.
    We vary the number of skills $z_i$ and use a fixed set of $z_i$ vectors.  This scheme is compared to the original version of sampling each time from the uniform $p(z)$. Using a small fixed set of pre-sampled skills is insufficient to cover the state space during pretraining or to yield identifiable representations, as shown by the peak coverage and corresponding $R^2$ score in \cref{fig:skill_diversity}. \Cref{sec:app_exp} shows further results for zero-shot task transfer.
    We also investigate how latent space dimensionality affects both identifiability and zero-shot task transfer---if the feature space is lower dimensional than the state space in the \gls{pomdp}, then all state information might be present but not linearly decodable.
    \Cref{fig:latent_dim} shows $R^2$ scores and zero-shot task transfer performance for varying latent space dimensionality. A low-dimensional latent space yields an information bottleneck that prevents the feature differences from encoding the ground truth states; however, CSF still performs well on task transfer.\footnote{CSF is designed for lower feature space dimensionality, and we empirically found it to be sensitive to increasing this hyperparameter.} Nonetheless, having a latent space dimensionality equal to the state dimensionality is a local maximizer of the oracle return. We show the effect of latent dimensionality on state coverage in \Cref{sec:app_exp}.

\section{Discussion}
\label{sec:discussion}

    \paragraph{Limitations.}
        Our work connected identifiability insights from nonlinear \acrfull{ica} to \acrfull{misl} and formulated practical insights. Our result is based on the observations of \citet{zheng_can_2024}, showing that the assumptions of nonlinear \gls{ica} theory can be applied to some \gls{rl} environments. However, it requires further research whether similar results hold in a broader range of scenarios.

    \paragraph{Extension to related works.}
        Our work is of an explanatory nature, advancing our understanding of \acrfull{misl} methods in \acrfull{rl}. To the best of our knowledge, we are the first to prove identifiability of the learned features in \gls{rl} (\cref{assum:rl_ident,thm:ident_theo_supervised}), particularly, for \acrfull{csf}~\citep{zheng_can_2024}. We accomplish this by showing that the modeling assumptions and design choices in \gls{csf} match those in the nonlinear \gls{ica} literature, particularly in DIET~\citep{ibrahim_occams_2024,reizinger_cross-entropy_2024}. We connect the notion of policy diversity to similar conditions both in the \gls{ica} and \acrfull{crl} literatures.
        Furthermore, our insights help explain some contributing factors of why \gls{misl} methods improved in the last few years. Namely, they introduced a linear parametrization of the learned discriminative model to distinguish the skills. The explicit connection to \acrfull{cl} by ~\citet{eysenbach_contrastive_2022} also helped this progress. Finally, design choices regarding the \acrfull{mi} objective and omitting maximum entropy regularizers also contributed to improved performance (\cref{subsec:insights}).
    \paragraph{Conclusion.}
        Our work theoretically proves that learning diverse skills is a meaningful surrogate objective in \acrfull{rl} for learning the ground-truth features of the environment up to a linear transformation. We show this by connecting the \acrfull{misl} family to nonlinear \acrfull{ica} methods, and proving linear identifiability for the features learned by \acrfull{csf}~\citep{zheng_can_2024}. Our identifiability guarantees not only provide a possible explanation of why \gls{misl} works, but also identify the key components of successful \gls{misl} methods. Furthermore, our theoretical insights help elucidate some failure modes of previous methods, \eg, the use of strong entropy regularizers of the policy. We hope that our insights will open up new research possibilities and also help practical algorithm design.

%% file: figures/fig1.tex
  
\tikzset {_0yo2k63xm/.code = {\pgfsetadditionalshadetransform{ \pgftransformshift{\pgfpoint{89.1 bp } { -108.9 bp }  }  \pgftransformscale{1.32 }  }}}
\pgfdeclareradialshading{_a4w4s76cw}{\pgfpoint{-72bp}{88bp}}{rgb(0bp)=(0.82,0.01,0.11);
rgb(0.08928571428571429bp)=(0.82,0.01,0.11);
rgb(25bp)=(1,1,1);
rgb(400bp)=(1,1,1)}
\tikzset{_nefk3g9na/.code = {\pgfsetadditionalshadetransform{\pgftransformshift{\pgfpoint{89.1 bp } { -108.9 bp }  }  \pgftransformscale{1.32 } }}}
\pgfdeclareradialshading{_dl703t3yl} { \pgfpoint{-72bp} {88bp}} {color(0bp)=(transparent!29);
color(0.08928571428571429bp)=(transparent!29);
color(25bp)=(transparent!34);
color(400bp)=(transparent!34)} 
\pgfdeclarefading{_psr4id62g}{\tikz \fill[shading=_dl703t3yl,_nefk3g9na] (0,0) rectangle (50bp,50bp); } 

  
\tikzset {_03hs4cs6b/.code = {\pgfsetadditionalshadetransform{ \pgftransformshift{\pgfpoint{89.1 bp } { -108.9 bp }  }  \pgftransformscale{1.32 }  }}}
\pgfdeclareradialshading{_kowdvsgts}{\pgfpoint{-72bp}{88bp}}{rgb(0bp)=(0.29,0.56,0.89);
rgb(0.08928571428571429bp)=(0.29,0.56,0.89);
rgb(25bp)=(1,1,1);
rgb(400bp)=(1,1,1)}
\tikzset{_392g5y07h/.code = {\pgfsetadditionalshadetransform{\pgftransformshift{\pgfpoint{89.1 bp } { -108.9 bp }  }  \pgftransformscale{1.32 } }}}
\pgfdeclareradialshading{_wzlb8c0aj} { \pgfpoint{-72bp} {88bp}} {color(0bp)=(transparent!40);
color(0.08928571428571429bp)=(transparent!40);
color(25bp)=(transparent!34);
color(400bp)=(transparent!34)} 
\pgfdeclarefading{_sfh3s6qz2}{\tikz \fill[shading=_wzlb8c0aj,_392g5y07h] (0,0) rectangle (50bp,50bp); } 
\tikzset{every picture/.style={line width=0.75pt}} 

\begin{tikzpicture}[x=0.75pt,y=0.75pt,yscale=-1,xscale=1,  every node/.style={font=\footnotesize}, >={Latex}]

\draw  [draw opacity=0][shading=_a4w4s76cw,_0yo2k63xm,path fading= _psr4id62g ,fading transform={xshift=2}] (106.89,93.82) .. controls (110.44,97.85) and (105.21,108.26) .. (95.2,117.08) .. controls (85.19,125.89) and (74.2,129.77) .. (70.65,125.74) .. controls (67.1,121.71) and (72.34,111.3) .. (82.34,102.48) .. controls (92.35,93.67) and (103.34,89.79) .. (106.89,93.82) -- cycle ;
\draw  [draw opacity=0][shading=_kowdvsgts,_03hs4cs6b,path fading= _sfh3s6qz2 ,fading transform={xshift=2}] (40.35,52.91) .. controls (40.35,45.29) and (51.16,39.12) .. (64.5,39.12) .. controls (77.84,39.12) and (88.65,45.29) .. (88.65,52.91) .. controls (88.65,60.53) and (77.84,66.71) .. (64.5,66.71) .. controls (51.16,66.71) and (40.35,60.53) .. (40.35,52.91) -- cycle ;

\draw (64.5,83.62) circle [radius=34pt];

\node[circle, fill=figblue, inner sep=1.5pt,] (circ1) at (51.39,52.91) {};
\node[circle, fill=figblue, inner sep=1.5pt,] (circ2) at (75.39,44.91) {};

\node[draw, rectangle, align=center, inner sep=2pt, fill=gray!60] (skill) at (133,81) {
  skills\\$p(z)$
};
\node[draw, rectangle, align=center, inner sep=2pt, fill=figgreen!70] (policy) at (210,81) {
  policy\\$\pi(a \mid \phi(o), z)$
};
\node[draw, rectangle, align=center, inner sep=2pt, fill=gray!60] (env) at (300,81) {
  environment\\$p(s' \mid s, a)$
};
\node[draw, rectangle, align=center, inner sep=2pt, fill=figgreen!80] (phi) at (380,81) {
  encoder\\$\phi(o),\phi(o')$
};
\node[align=center, inner sep=2pt] (critic) at (480,81) {
  critic\\$q(z_i|\phi(o), \phi(o'))$
};
\node[align=center] (reward) at (210,31) {
  reward: $\max r_z(\phi(o),\phi(o'))$
};

\path[->, thick] (skill) edge (policy);
\path[->, thick] (policy) edge (env);
\path[->, thick] (env) edge (phi);
\path[->, thick] (phi) edge (critic);
\path[->, thick] (policy) edge (reward);

\draw[->, thick, figblue] (64.5,83.62) -- node[near start,left] (zi) {$z_i$} (64.5,39.12); 
\draw[->, thick, figred] (64.5,83.62) -- node[left] {$z_j$} (95.2,117.08); 

\draw[->, thick, bend right=45] 
  ([xshift=4pt]env.east) 
  --
  ([yshift=25pt,xshift=4pt]env.east) 
  -- node[above] {$s'$}
  ([yshift=25pt,xshift=10pt]skill.east) 
  -- 
  ([xshift=10pt]skill.east);

\draw[<-, thick, dashed, figblue!90] (zi) .. controls (180.66,60.12) and (287.74,33.56) .. node[pos=0.85, above] {$(\phi(o), \phi(o'))\mapsto z_i$} (critic);


\end{tikzpicture}

%% file: figures/mi_insights.tex
\tikzset{every picture/.style={line width=0.75pt}} 

\begin{tikzpicture}[x=0.75pt,y=0.75pt,yscale=-.71,xscale=.71, >={Latex}, every node/.style={font=\scriptsize}]

\draw    (341,30) -- (341,230) ;

\node[draw, circle, minimum size=68pt, label={[above=5pt]{$I(s,s';z) \ : \ [ \phi(o) - \phi(o') ]^{\top} z$}}] (left) at (201.5,137.5) {};
\node[draw, circle, minimum size=68pt, label={[above=5pt]{$I(s;z) \ : \  \phi(s)^{\top} z$}}] (right) at (473.5,136.5) {};


\draw[->, thick, figred]
  (left.center) -- node[midway, below, text=figred] {$z_i$}
  ++(61.5, 11.9);


\draw[->, thick, figblue]
  ([yshift=-6pt]left.center) -- node[midway, above left, text=figblue] {$\phi(o') - \phi(o)$}
  ++(61.5, 11.9);

\draw[->, thick, figblue]
  (right.center) -- node[midway, above, text=figblue] {$\phi(o)$}
  ++(-62.5, -13);

\draw[->, thick, figred]
  ([yshift=-4.5pt]right.center) -- node[midway, above, text=figred] {$z_i$}
  ++(61.5, 11.9);

\draw[->, thick, figblue]
  (right.center) -- node[midway, below, text=figblue] {$\phi(o')$}
  ++(61.5, 11.9);

\draw (150,215) node [anchor=north west][inner sep=0.75pt]   {${\color{figblue}\phi (o') -\phi (o)} \ || \color{figred}\ z_{i} \ $};
\draw (422,215) node [anchor=north west][inner sep=0.75pt]  {${\color{figblue}\phi (o) ,\ \phi ( o') }\ || \color{figred}\ z_{i} \ $};

\end{tikzpicture}

%% file: appendix.tex
\section{Impact Statement}\label{sec:impact}

This paper presents work whose goal is to advance the fields of
Reinforcement Learning and Identifiable Representation Learning. Our focus on representation identifiability promotes transparency and interpretability, which are important safeguards against unintended use.

\section{Proofs}
    \subsection{Affine Generator Systems}\label{subsec:app_assum}

    \begin{definition}[Affine Generator System~\citep{reizinger_cross-entropy_2024} (Defn.~1) ]\label{def:affine_generator}
    A system of vectors $\braces{z_i \in \rr{d}}$ is called an \emph{affine generator system} if any vector in $\rr{d}$ is an affine linear combination of the vectors in the system. Put into symbols: for any $z_i\in \rr{d}$ there exist coefficients $\alpha_i \in \rr{}$, such that
    \begin{equation}
        z = \sum_{i} \alpha_i z_i \quad\text{and}\quad \sum_{i} \alpha_i = 1.
    \end{equation}
    \end{definition}

    \begin{lem}[Properties of affine generator systems~\citep{reizinger_cross-entropy_2024} (Lem.~1)]
        \label{lem:affine_generator}
        The following hold for any affine generator system $\braces{z_i \in \rr{d} }$:
        \begin{enumerate}[leftmargin=*]
            \item for any $i\neq j$ the system $\braces{z_i -z_j }$ is now a generator system of $\rr{d}$;
            \item the invertible linear image of an affine generator system is also an affine generator system.
        \end{enumerate}
    \end{lem}

\section{Experimental Details and Further Results}\label{sec:app_exp}

    \subsection{Compute Resources}\label{subsec:compute}

    All experiments were run in a compute cluster using an Intel Xeon Gold CPU ($16$ cores, $2.9$ GHz) and NVIDIA RTX 2080 Ti GPUs and used at most $48$ GB of RAM. No experiment required more than $3$ days of runtime. In total, our experiments took $0.4$ GPU years.

    \subsection{Hyperparameter Search}

    To train the CSF method in the Ant, Half Cheetah, Quadruped State, and Quadruped Pixel environments, we used the hyperparameters included in the \href{https://github.com/Princeton-RL/contrastive-successor-features}{GitHub repository} of~\citet{zheng_can_2024} as a starting point and modified (i) the encoder $\phi$'s backbone architecture to aid identfiability by introducing skip-connections, (ii) the latent space dimensionality to match the ground truth state's dimensionality, (iii) the trade-off factor $\xi$ between the two factors of the contrastive loss~\citep[Eq. (10) \& paragraph below]{zheng_can_2024}, and (iv) the number of negative samples in the contrastive loss. We found the CSF method to be sensitive to the latent space dimensionality, but an appropriate choice of $\xi$ mitigated this sensitivity and encouraged learning. Increasing the number of negative samples in the contrastive loss led to performance gains in some environments (w.r.t.~state space coverage and oracle reward), which is intuitively explained by requiring more samples to cover the latent space well.

    In our experiments that varied the number of skills in the discrete skill set, we kept all other hyperparameters fixed to the continuous sampling case. In the experiments varying the latent space dimensionality, the other hyperparameters were fixed to those of $\text{dim} = 29$.

    For the exact hyperparameter configurations, see the code in the supplementary material.

    \subsection{Reproducibility}

    We provide shell scripts to reproduce our results in the \texttt{scripts} folder of the code in the supplementary material. To obtain error bars, we additionally varied the seed parameter in the scripts to obtain three to five independent runs.

    \subsection{Further Results}

    \cref{fig:skill_diversity_skill_transfer} shows the counterpart of \cref{fig:skill_diversity} where the zero-shot skill transfer performance is evaluated instead of state space coverage. Similarly, \cref{fig:latent_dim_coverage} shows a variant of \cref{fig:latent_dim} with state space coverage reported instead of zero-shot skill transfer performance.

    \begin{figure}
        \centering
        \includegraphics{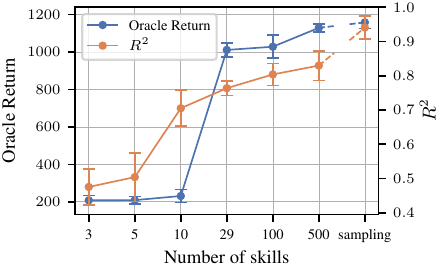}
        \caption{\textbf{The effect of skill diversity on state identifiability and zero-shot skill transfer in the Ant environment.} Skills are sampled from $p(z)$ at the start of pretraining and kept fixed throughout, except in the `sampling' case where skills are redrawn from $p(z)$ during training, emulating an infinite set of skills. An insufficient number of skills violates \cref{assum:rl_ident}(i), leading to both weaker zero-shot skill transfer and a lower \gls{r2} score.}
        \label{fig:skill_diversity_skill_transfer}
    \end{figure}

    \begin{figure}
        \centering
        \includegraphics{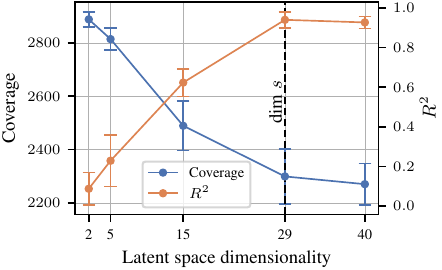}
        \caption{\textbf{The effect of latent space dimensionality on state identifiability and coverage in the Ant environment.} Linear identifiability requires that the feature space has at least as many dimensions as the true state (\cf \cref{assum:rl_ident}(v)).}
        \label{fig:latent_dim_coverage}
    \end{figure}

\newpage
\printglossaries